\documentclass{article}

\usepackage{microtype}
\usepackage{graphicx}
\usepackage{subcaption}
\usepackage{booktabs} 
\usepackage{pifont}
\usepackage{hyperref}
\usepackage{multirow}

\usepackage[preprint]{icml2026}

\usepackage{amsmath}
\usepackage{amssymb}
\usepackage{mathtools}
\usepackage{amsthm}

\usepackage[capitalize,noabbrev]{cleveref}

\theoremstyle{plain}
\newtheorem{theorem}{Theorem}[section]
\newtheorem{proposition}[theorem]{Proposition}
\newtheorem{lemma}[theorem]{Lemma}

\theoremstyle{definition}

\theoremstyle{remark}

\usepackage[textsize=tiny]{todonotes}

\icmltitlerunning{ACPO based on CCI Framework for Offline RL}

\begin{document}

\twocolumn[
  \icmltitle{Automatic Constraint Policy Optimization based on Continuous Constraint Interpolation Framework for Offline Reinforcement Learning}



  \icmlsetsymbol{equal}{*}

  \begin{icmlauthorlist}
    \icmlauthor{Xinchen Han}{ipp}
    \icmlauthor{Qiuyang Fang}{nankai}
    \icmlauthor{Hossam Afifi}{ipp}
    \icmlauthor{Michel Marot}{ipp}
  \end{icmlauthorlist}

  \icmlaffiliation{ipp}{SAMOVAR, Institut Polytechnique de Paris, Paris, France}
  \icmlaffiliation{nankai}{College of Artificial Intelligence, Nankai University, Tianjin, China}

  \icmlcorrespondingauthor{Xinchen Han}{xinchen.han@telecom-sudparis.eu}

  \icmlkeywords{Machine Learning, ICML}

  \vskip 0.3in
]



\printAffiliationsAndNotice{}  

\begin{abstract}
Offline Reinforcement Learning (RL) relies on policy constraints to mitigate extrapolation error, where both the constraint form and constraint strength critically shape performance. However, most existing methods commit to a single constraint family: weighted behavior cloning, density regularization, or support constraints, without a unified principle that explains their connections or trade-offs. In this work, we propose Continuous Constraint Interpolation (CCI), a unified optimization framework in which these three constraint families arise as special cases along a common constraint spectrum. The CCI framework introduces a single interpolation parameter that enables smooth transitions and principled combinations across constraint types. Building on CCI, we develop Automatic Constraint Policy Optimization (ACPO), a practical primal--dual algorithm that adapts the interpolation parameter via a Lagrangian dual update. Moreover, we establish a maximum-entropy performance difference lemma and derive performance lower bounds for both the closed-form optimal policy and its parametric projection. Experiments on D4RL and NeoRL2 demonstrate robust gains across diverse domains, achieving state-of-the-art performance overall.
\end{abstract}

\section{Introduction}
RL tackles sequential decision-making through online interaction with an environment \cite{LeakyPPO}. However, in many real-world applications, such interaction can be expensive, risky, or even infeasible. To address this limitation, offline RL, which aims to learn effective policies from pre-collected datasets without further interaction, has received growing attention in recent years. Yet, removing interaction also introduces fundamental challenges. In particular, extrapolation error \cite{BCQ} arises when $Q$-function estimates are queried on out-of-distribution (OOD) actions that are poorly supported by the dataset. Such errors can be amplified by bootstrapping, leading to severe overestimation and, ultimately, degraded performance and instability \cite{ORLtutorial, CQL, AEM, PARS}.

A dominant paradigm of offline RL methods mitigates extrapolation error by constraining the learned policy to stay close to the behavior policy that generated the dataset. In practice, however, the behavior policy is often unknown: offline data may come from random exploration, mixtures of multiple tasks, heterogeneous sources, or logged trajectories collected during online RL training \cite{LAPO, MCQ}. Under such ambiguity, selecting an appropriate form and strength of policy constraint becomes crucial.

From a high-level perspective, mainstream policy-constrained approaches in offline RL can be organized into three categories along a common constraint spectrum: weighted behavior cloning (wBC), density regularization, and support constraints. wBC remains fundamentally imitative and typically reduces to weighted supervised learning on the dataset \cite{BAIL, IBC, DWBC, TD3+BC}. Density regularization allows controlled distributional deviation from the behavior policy. In particular, KL-based density regularization is widely used to limit the magnitude of distributional shift from the behavior distribution \cite{RWR, AWR, AWAC, CRR}. Support constraints, instead of matching the behavior distribution, encourage the learned policy to select actions within the support of the behavior policy \cite{BEAR, PIQL, SPOT, STR}. However, overly conservative constraints can stabilize policy evaluation but undermine the core promise of offline RL---policy improvement beyond the dataset; while overly weak constraints fail to meaningfully restrict policy improvement, tending to select OOD actions and inducing large value misestimation that misguides optimization \cite{ACLQL, CosIQL}.

To cope with this trade-off, several recent works \cite{ABR, SSAR, PIQL} explore adaptive constraint schemes. However, existing approaches typically adapt within a single constraint paradigm, and a unified view that simultaneously covers and connects wBC, KL-based density regularization, and support constraints remains missing. As a result, practitioners still need to manually choose the constraint form and then tune its strength, which can be nontrivial across datasets and tasks.

In this paper, we develop a unified optimization framework, the \emph{Continuous Constraint Interpolation (CCI)} framework, 
and show that the above three constraint families arise as special cases. The CCI framework exposes a single interpolation parameter that enables smooth transitions and principled combinations among wBC, KL-based density regularization, and support constraints for policy improvement. Building on the CCI framework, we propose \emph{Automatic Constraint Policy Optimization (ACPO)}, a practical offline RL algorithm that optimizes the interpolation parameter via a Lagrangian dual update, thereby eliminating the need to manually select fixed constraint forms and strengths. Intuitively, the interpolation parameter acts as a constraint ``dial'' that automatically adapts the constraints regime during training.

Our key contributions can be summarized as follows:

\begin{itemize}
    \item \textbf{Unified CCI Framework.} We propose CCI, a unified optimization framework that bridges three major policy-constraint families: wBC, KL-based density regularization, and support constraints, and enables smooth transitions as well as principled combinations across constraint regimes.
    \item \textbf{Practical ACPO Algorithm.} Building on CCI, we develop ACPO, a practical primal--dual algorithm that automatically adapts the interpolation parameter through a Lagrangian dual update. We further establish a maximum-entropy performance difference lemma and derive performance lower bounds for both the optimal policy and its practical parametric projection.
    \item \textbf{Empirical Evaluation.} We perform extensive comparisons and ablations on the standard D4RL and the recent NeoRL2 benchmarks, showing that ACPO achieves consistent improvements across diverse domains and tasks.
\end{itemize}

\section{Related Works}
In offline RL, policy constraints are not the only mechanism for mitigating extrapolation error induced by OOD actions. Beyond policy constraints, prior work has explored conservative value estimation \cite{CQL, CEN, APTQ}, ensemble-based uncertainty-aware methods \cite{EDAC, ACTIVE, PBRL}, and DIstribution Correction Estimation (DICE) algorithms \cite{DualDICE, OptiDICE, DiffDICE}. To highlight the advantages of our framework and algorithm, we focus the remainder of this section on policy-constrained offline RL and organize representative methods into three categories: wBC, density regularization, and support constraints.

wBC methods either cast policy learning as (weighted) supervised imitation on the dataset, or incorporate an explicit BC regularizer into the policy improvement objective. BAIL \cite{BAIL} learns a value function and imitates only actions selected as high-quality under the learned value. IBC \cite{IBC} improves expressiveness via implicit energy-based policies, avoiding the limitations of explicit regression-based models. DWBC \cite{DWBC} trains a discriminator to estimate expert-likeness and uses it to reweight the BC loss. In contrast, TD3+BC \cite{TD3+BC} augments value-based policy improvement with a BC term, mitigating distributional drift during policy improvement.

Density regularization methods allow controlled deviation from the behavior policy by penalizing distributional shift at the policy level. The most common choice is KL-based regularization \cite{RWR, AWR, AWAC, CRR, IQL}, due to its elegant and tractable policy improvement objective, which keeps policy updates via stable in-sample learning. For completeness, we also review density regularization choices beyond KL divergence, including MMD-based constraints for sample-based distribution matching \cite{BRAC}, Fisher divergence \cite{Fisher-BRC}, and Wasserstein distance regularization \cite{Wasserstein, OT}.

Support constraint methods perform policy improvement with the policy constrained within the support of the behavior policy. EMaQ \cite{EMaQ} samples multiple candidate actions from behavior policy and performs backups over these in-support proposals. STR \cite{STR} combines trust-region policy optimization with an explicit support constraint via Importance Sampling (IS), ensuring that each update remains within the behavior support. SPOT \cite{SPOT} trains a VAE-based estimator to model the behavior support and introduces a simple plug-in regularization term that keeps policy improvement within the estimated support.

Overall, most existing approaches operate within a single policy-constraint paradigm. A unified optimization view that bridges all three families under one framework---enabling smooth transitions, flexible combinations, and automatic adaptation across constraint regimes---remains missing. We therefore propose the CCI framework and derive a practical algorithm, ACPO.


\section{Preliminaries}

\subsection{Maximum Entropy RL}
Conventionally, an RL problem is defined as a Markov Decision Process (MDP) \cite{SuttonRL}, specified by a tuple $\mathcal{M} = \langle \mathcal{S}, \mathcal{A}, \mathcal{T}, r, \gamma \rangle$, where $\mathcal{S}$ and $\mathcal{A}$ denote the state and action spaces, $\mathcal{T}$ represents the transition function, $r$ is the reward function, and $\gamma \in [0,1)$ denotes the discount factor. The performance difference lemma \cite{PDT} in the standard RL setting gives
\begin{gather} 
\begin{aligned}
\eta(\pi') - \eta(\pi)
= \frac{1}{1-\gamma}\sum_{s} d_{\pi'}(s)\sum_{a}\pi'(a|s)\,A^{\pi}(s,a),
\label{Eq-Per-Dif-Standard}
\end{aligned}
\end{gather}
where $\eta (\pi)=\mathbb{E}_{\tau\sim\pi}\left[ \sum_{t=0}^{\infty}\gamma^t r(s_t,a_t) \right]$, and $\mathbb{E}_{\tau\sim\pi}[...]$ denotes the expectation under the trajectory $\tau$ induced by the interconnection of $\pi$ and the environment. $d_{\pi'}(s)$ denotes the discounted state visitation distribution under $\pi'$, and $A^{\pi}(s, a) = Q^{\pi}(s, a) - V^{\pi}(s)$ is the advantage function. 

In the maximum-entropy setting~\cite{SQL, MERL, SAC, MEMC}, the objective augments rewards with an entropy term:
\begin{gather}
\begin{aligned}
\mathcal{J}(\pi)
= \mathbb{E}_{\tau\sim\pi}\!\left[\sum_{t=0}^{\infty}\gamma^t\Big(r(s_t,a_t)-\alpha\log\pi(a_t|s_t)\Big)\right],
\end{aligned}
\end{gather}
where $\alpha>0$ is the temperature parameter and $\mathcal{H}(s,\pi)=\mathbb{E}_{a\sim\pi(\cdot|s)}[-\log\pi(a|s)]$.
The maximum-entropy value function is defined as $V^{\pi}(s) = \mathbb{E}_{a\sim\pi(\cdot|s)} \left[ Q^{\pi}(s,a)-\alpha\log\pi(a|s) \right]$. The soft $Q$-function is $Q^{\pi}(s,a) = r(s,a) + \gamma \mathbb{E}_{s'\sim\mathcal{T}(\cdot|s,a)} \left[V^{\pi}(s') \right]$. As $\alpha \rightarrow 0$, maximum entropy RL recovers the standard RL.

\subsection{wBC Methods} 
The generalized wBC objective can be defined as follows,
\begin{gather} 
\begin{aligned}
\mathop{\max}_{\pi} \mathcal{J}^{\text{wBC}}(\pi) = \mathbb{E}_{(s, a) \sim \mathcal{D}} \left[w^{\text{bc}}(s, a)\log \pi(a|s) \right],
\label{Eq-wBC-Obj}
\end{aligned}
\end{gather}
where $w^{\text{bc}}: \mathcal{S} \times \mathcal{A} \rightarrow \mathbb{R}_{+}$ denotes a weight function. For example, $w^{\text{bc}}(s, a) = 1$ corresponds the vanilla BC method, while more sophisticated choices are used in, e.g., \cite{DWBC, BAIL}.

\subsection{KL-based Density Regularization Methods} 
KL-based density regularization methods can be summarized by the following policy optimization problem \cite{STR}:
\begin{gather} 
\begin{aligned}
\pi^{*}(a|s) & = \mathop{\arg\max}_{\pi} \mathbb{E}_{a \sim \pi(\cdot|s)} \left[ A^{\pi_{\text{pe}}}(s, a)\right], \\
\text{s.t.} \; & \mathop{D_{KL}} \left[\pi(\cdot|s)||\pi_{\text{base}}(\cdot|s) \right] \leq \epsilon, \\
& \int_{a} \pi(a|s) da = 1,
\label{Eq-KL-Density-Obj-Pro}
\end{aligned}
\end{gather}
where $\pi_{\text{base}}$ denotes a reference (sampling) policy and $\pi_{\text{pe}}$ denotes the evaluation policy. The optimization problem Eq.~\eqref{Eq-KL-Density-Obj-Pro} admits a parametric solution, which can be written as 
\begin{gather} 
\begin{aligned}
\mathop{\max}_{\pi} \mathcal{J}^{\text{KL}}(\pi) = \mathbb{E}_{s \sim \mathcal{D}, a \sim \pi_{\text{base}}} \left[ w^{\text{kl}}(s, a) \log \pi(a|s) \right],
\label{Eq-KL-Density-Obj}
\end{aligned}
\end{gather}
where $w^{\text{kl}}(s, a) = \exp\left(\frac{A^{\pi_{\text{pe}}}(s, a)}{\alpha}\right)$. Different choices of $\pi_{\text{base}}$ and $\pi_{\text{pe}}$ give rise to different algorithmic instantiations~\cite{AWAC, CRR, ABM}, while sharing the same KL-regularized principle.

\subsection{Support Constraint Methods}
The support constraint requires the learned policy to assign positive density only on actions supported by the behavior policy, e.g., $\pi_{\beta}(a|s)=0 \ \Rightarrow\ \pi(a|s)=0$.

InAC \cite{InAC} introduces an in-sample softmax operator and implements support-constrained policy improvement by sampling actions from the dataset:
\begin{gather} 
\begin{aligned}
\mathop{\max}_{\pi} \mathcal{J}^{\text{Supp}}(\pi) = \mathbb{E}_{\substack{(s, a)\sim \mathcal{D}}} \left[ w^{\text{supp}}(s, a) \log \pi(a|s) \right],
\label{Eq-Supp-Obj}
\end{aligned}
\end{gather}
where $w^{\text{supp}}(s, a) = \exp\left(\frac{A^{\pi}(s, a)}{\alpha} - \log \pi_{\beta}(a|s) \right)$.

Because $\exp (-\log \pi_{\beta}(a|s)) = \pi_{\beta}(a|s)^{-1}$, the above is equivalent to $\pi(a|s) = 0$ when $\pi_{\beta}(a|s) = 0$ and otherwise $\pi(a|s) \propto \exp (A^{\pi}(s, a) / \alpha)$ \footnote{InAC defines $0 \cdot \infty = 0.$}. The learned policy $\pi$ is not skewed by the action probabilities in $\pi_{\beta}$; it just has the same support \cite{InAC}.


\section{Unified Policy Constraints for Offline RL}

\subsection{Continuous Constraint Interpolation Framework}
The CCI framework provides a unified optimization view of three policy-constraint families: wBC, KL-based density regularization, and support constraints. It treats them as instances along a common constraint spectrum, and enables smooth switching and principled combinations via a single continuous interpolation parameter.

Consider the following constrained maximum-entropy policy optimization problem,
\begin{gather}
\begin{aligned}
\max_{\pi}\ \ & \mathbb{E}_{s\sim \mathcal{D},\,a\sim \pi(\cdot|s)}
\big[ Q(s,a) - \alpha \log \pi(a|s) \big],\\
\text{s.t.}\ \ &
\mathbb{E}_{s\sim \mathcal{D},\,a\sim \pi(\cdot|s)}
\big[ \log \pi_{\beta}(a|s) \big] \ge \epsilon, \\
& \int \pi(a|s)\, da = 1,\ \forall s.
\label{Eq-CCI-Opt-Pro}
\end{aligned}
\end{gather}
We do not impose per-state hard constraints, which would introduce an impractical (potentially infinite) number of constraints~\cite{SPOT}. Instead, we enforce the constraint on average over dataset states, yielding a single scalar constraint.

\begin{table*}[!t] 
\small
\centering
\setlength{\tabcolsep}{6pt}
\renewcommand{\arraystretch}{1.15}
\begin{tabular}{lll}
\toprule
$\lambda$ & Constraint & $w_{\lambda}(s,a)$ \\
\midrule

$\lambda = 0$ 
& Support Constraint
& $\exp \big(\frac{A(s,a)}{\alpha} - \log \pi_{\beta}(a|s)\big)$.\\

$0 < \lambda < \alpha$ 
& Support-to-Density Constraint
& $\exp \big( \frac{A(s,a)}{\alpha} - \frac{\alpha - \lambda}{\alpha}\log \pi_{\beta}(a|s)\big)$. \\

$\lambda = \alpha$
& KL-based Density Constraint
& $\exp \big(\frac{A(s,a)}{\alpha} \big)$. \\

$\lambda > \alpha$ 
& Density-to-wBC Constraint
& $\exp \big( \frac{A(s,a)}{\alpha} + \frac{\lambda - \alpha}{\alpha}\log \pi_{\beta}(a|s)\big)$. \\

$\lambda \gg \alpha + A(s,a) (\log \pi_{\beta}(a|s))^{-1}$ \qquad \qquad 
& Practical wBC Constraint \qquad \qquad \qquad 
& $ \approx \exp \big( \frac{\lambda-\alpha}{\alpha}  \log \pi_{\beta}(a|s) \big)$. \\
\bottomrule
\end{tabular}
\caption{Constraint spectrum induced by the interpolation parameter $\lambda$ in the CCI framework.}
\label{Tab-CCI-Lambda}
\end{table*}

Solving the above problem Eq.~\eqref{Eq-CCI-Opt-Pro} via the Lagrangian, 
\begin{gather}
\begin{aligned}
\mathcal L(\pi,\lambda,\nu)
= &\ \mathbb{E}_{s\sim \mathcal{D},\,a\sim \pi(\cdot|s)}
\left[ Q(s,a) - \alpha \log \pi(a|s) \right] \\
& + \lambda\Big(\mathbb{E}_{s\sim \mathcal{D},\,a\sim \pi(\cdot|s)}
\left[\log \pi_{\beta}(a|s)\right] - \epsilon\Big) \\
& + \mathbb{E}_{s\sim\mathcal D}\left[\nu(s)\Big(\int \pi(a|s)\, da - 1\Big)\right],
\label{Eq-CCI-Lagrangian}
\end{aligned}
\end{gather}
we obtain the following closed-form optimal policy:
\begin{gather} 
\begin{aligned}
\pi^*_{\lambda}(a|s) = \exp \Big( \frac{Q(s, a) - Z_{\lambda}(s)}{\alpha} \Big) \pi_{\beta}(a|s)^{\lambda / \alpha}.
\label{Eq-CCI-ClosedForm}
\end{aligned}
\end{gather}
where $\lambda \ge 0$ denotes the dual variable, and
$Z_\lambda(s)= \alpha \log \int \exp(Q(s,a)/\alpha)\,\pi_\beta(a|s)^{\lambda/\alpha}\,da$
is the log-normalizer.

Then, we project the non-parametric optimal policy $\pi_{\lambda}^*$ into the parametric space $\pi_{\theta}$ by minimizing the reverse KL divergence,
\begin{gather} 
\begin{aligned}
\pi_{\theta}(a|s) = \mathop{\arg\min}_{\pi_{\theta}} \mathbb{E}_{s\sim \mathcal{D}} \big[\mathop{D_{\mathrm{KL}}} \left[\pi^*_{\lambda}(\cdot|s) || \pi_{\theta}(\cdot|s) \right] \big].
\label{Eq-CCI-Reverse-KL}
\end{aligned}
\end{gather}

This yields the following policy improvement objective under CCI:
\begin{gather} 
\begin{aligned}
\max_{\theta}\ \mathcal{J}^{\text{CCI}}(\theta;\lambda) = \mathbb{E}_{(s,a)\sim \mathcal{D}} \left[ w_{\lambda}(s, a) \log \pi_{\theta}(a|s) \right],
\label{Eq-CCI-Obj}
\end{aligned}
\end{gather}
where $w_{\lambda}(s, a) = \exp \left(\frac{A(s,a)} {\alpha} + \frac{\lambda-\alpha}{\alpha}\log \pi_{\beta}(a|s)\right)$.

Based on the policy improvement objective Eq.~\eqref{Eq-CCI-Obj}, varying $\lambda$ enables continuous interpolation of the constraint spectrum, allowing smooth transitions and flexible combinations, shown in Tab.~\ref{Tab-CCI-Lambda}.

We note that the practical wBC regime corresponds to choosing
$
\lambda \gg \alpha + A(s,a)\big(\log \pi_{\beta}(a|s)\big)^{-1}.
$
Under the maximum-entropy setting, assume rewards are bounded by $|r(s,a)|\le R_{\max}$ and we stabilize training by clipping log-probabilities for any stochastic policy as $\log\pi(a|s)\in[\ell_{\min},\ell_{\max}]$.
Let 
$$C_{\max}\triangleq \max\{|\ell_{\min}|,|\ell_{\max}|\} \Rightarrow |-\alpha\log\pi(a|s)|\le \alpha C_{\max}.$$
Then one can obtain a uniform advantage bound
$
|A^{\pi}(s,a)|\le \frac{2\big(R_{\max}+\alpha C_{\max}\big)}{1-\gamma}
$
and consequently, with $C^{\beta}_{\min} \triangleq \min |\log\pi_\beta(a|s)|$ and a small constant $\delta>0$ to avoid division by zero,
\begin{equation}
\left|\frac{A(s,a)}{\log\pi_\beta(a|s)}\right|
\le
\frac{2\big(R_{\max}+\alpha C_{\max}\big)}{(1-\gamma)(C^{\beta}_{\min} + \delta)}.
\label{Eq:A_over_logpb_bound_maxent_main}
\end{equation}
Therefore, a sufficient condition for the practical wBC constraint is
\begin{equation}
\lambda \gg \alpha + \frac{2\big(R_{\max}+\alpha C_{\max}\big)}{(1-\gamma)(C^{\beta}_{\min} + \delta)}.
\label{Eq:wBC_sufficient_condition_main}
\end{equation}

The derivations for this subsection are provided in Appendix~\ref{Appendix-CCI}.

\subsection{Automatic Constraint Policy Optimization}
Building on the CCI framework, we propose ACPO, a practical primal--dual method that automatically tunes the constraint form and strength parameter $\lambda$.

For a fixed $Q$-function, the objective in Eq.~\eqref{Eq-CCI-Opt-Pro} is maximizing a concave objective in $\pi$, and the inequality constraint and normalization constraints are affine. Therefore, under standard regularity assumptions (ensuring $\mathcal{L}$ is well-defined) and Slater's condition, strong duality holds, and the primal problem is equivalent to the dual form,
\begin{gather} 
\begin{aligned}
\min_{\lambda} & \max_{\pi_{\theta}\in\Delta}\ \mathcal{L}(\pi_{\theta}, \lambda),\\
& \text{s.t.} \; \lambda \geq 0,
\label{Eq-CCI-Dual}
\end{aligned}
\end{gather}
where $\Delta$ denotes the set of valid policies.

ACPO alternates between (i) updating $\pi_\theta$ to maximize the CCI objective Eq.~\eqref{Eq-CCI-Obj} for a fixed $\lambda$, and (ii) performing a projected gradient descent step on $\lambda$ to enforce the constraint:
\begin{gather} 
\begin{aligned}
\lambda \leftarrow \Big[\lambda - \eta_\lambda\big(\mathbb{E}_{s\sim\mathcal D,\,a\sim\pi_\theta}[\log\pi_\beta(a|s)]-\epsilon\big)\Big]_+.
\label{Eq-lambda-descend}
\end{aligned}
\end{gather}

This update increases $\lambda$ when the constraint is violated, and decreases it when the constraint is overly conservative. 
By updating $\lambda$ online via Eq.~\eqref{Eq-lambda-descend}, ACPO automatically adapts the effective constraint regime to the dataset.

For policy evaluation, we adopt standard maximum-entropy critic updates with Clipped Double Q-Learning~\cite{ClipDoubleQ}:
\begin{gather}
\begin{aligned}
\mathcal{L}_V & (\psi) = \mathbb{E}_{s\sim\mathcal{D}, a\sim\pi_\theta(\cdot|s)} \left[\frac12\big(V_\psi(s)-Y_V(s,a)\big)^2\right], \\
& Y_V(s,a) \triangleq \min_{i\in\{1,2\}} Q_{\phi_i}(s,a) -\alpha\log\pi_\theta(a|s).
\end{aligned}
\label{Eq-Loss-V}
\end{gather}
\begin{gather} 
\begin{aligned}
\mathcal{L}_{Q}(\phi_i) = \mathbb{E}_{(s,a,r,s') \sim \mathcal{D}} \left[ \frac{1}{2}\left( r + \gamma V_{\psi}(s') - Q_{\phi_i}(s, a) \right)^2 \right].
\label{Eq-Loss-Q}
\end{aligned}
\end{gather}

Following common practice in offline RL \cite{InAC, PIQL}, we fit the behavior policy $\pi_{\beta}$ on the dataset via maximum likelihood:
\begin{gather} 
\begin{aligned}
\mathcal{L}(\pi_{\beta}) = \mathbb{E}_{(s,a) \sim \mathcal{D}} \left[ - \log \pi_{\beta}(a|s) \right].
\label{Eq-Behavior-Loss}
\end{aligned}
\end{gather}
Integrating policy improvement, policy evaluation, and dual optimization, the complete ACPO algorithm is summarized in Algorithm.~\ref{algorithm}.
\begin{algorithm}[!htbp]
\caption{ACPO}
\label{algorithm}
\textbf{Input}: Dataset $\mathcal{D} = \{(s, a, r, s')\}$.\\
\textbf{Parameters}: $Q$-networks $Q_{\phi_{1,2}}$, target $Q$-networks $Q_{\phi'_{1,2}}$, value network $V_\psi$, policy network $\pi_\theta$, behavior policy network $\pi_{\beta}$.
\begin{algorithmic}[1]
\STATE {\textbf{\emph{// Pre-train behavior policy for $N_\beta$ steps.}}}
\FOR{$t=1$ to $N_\beta$}
\STATE Optimize $\pi_{\beta}$ by Eq.~\eqref{Eq-Behavior-Loss}.
\ENDFOR
\STATE {\textbf{\emph{// Policy evaluation, improvement, and dual update.}}}
\FOR{each gradient step}
\STATE Sample $(s, a, r, s') \sim \mathcal{D}$.
\STATE Optimize $V_\psi$ by Eq.~\eqref{Eq-Loss-V}.
\STATE Optimize $Q_{\phi_i}$ by Eq.~\eqref{Eq-Loss-Q}.
\STATE Update $\lambda$ by Eq.~\eqref{Eq-lambda-descend}.
\STATE Optimize $\pi_\theta$ by Eq.~\eqref{Eq-CCI-Obj}.
\STATE Soft-update target networks $\phi'\leftarrow (1-\tau)\phi' + \tau\phi$.
\ENDFOR
\STATE \textbf{return} Policy $\pi_\theta$.
\end{algorithmic}
\end{algorithm}

\subsection{Theoretical Analysis}
The CCI framework admits a closed-form optimizer $\pi^*_\lambda$ parameterized by the dual variable $\lambda$. We first formalize how $\lambda$ controls conservatism by studying the constraint $\mathbb{E}_{a\sim\pi^*_\lambda(\cdot|s)}[\log\pi_\beta(a|s)]$.
We then derive a maximum-entropy performance difference lemma and establish performance lower bounds for both the optimal policy $\pi^*_\lambda$ and its parametric projection $\pi_\theta$.

\begin{proposition} \label{Proposition}
Fix any state $s$ and temperature $\alpha>0$. Let $\pi^*_{\lambda}(\cdot|s)$ denote the closed-form optimizer in Eq.~\eqref{Eq-CCI-ClosedForm} for a given $\lambda\ge 0$. Define 
\begin{equation} 
g_s(\lambda)
:= \mathbb{E}_{a\sim \pi^*_{\lambda}(\cdot|s)}\!\big[\log \pi_\beta(a|s)\big].
\label{Eq-dual-opt}
\end{equation}

Then $g_s(\lambda)$ is monotone non-decreasing in $\lambda$, and
\begin{equation} 
\frac{d}{d\lambda} g_s(\lambda)
= \frac{1}{\alpha}\,
\mathrm{Var}_{a\sim \pi^*_{\lambda}(\cdot|s)}\!\big(\log \pi_\beta(a|s)\big)
\;\ge\; 0.
\label{Eq-gs-prime}
\end{equation}

Consequently, the constraint $g(\lambda):=\mathbb{E}_{s\sim \mathcal{D}}[g_s(\lambda)]$ also satisfies $g'(\lambda)\ge 0$.
\end{proposition}

Proposition \ref{Proposition} shows the following intuition. As $\lambda$ increases, the closed-form optimal policy $\pi^*_{\lambda}$ assigns more probability mass to actions with larger $\pi_\beta$ and yields a more conservative update. Conversely, a smaller $\lambda$ allows $\pi^*_{\lambda}$ to retain greater flexibility within the support of $\pi_\beta$ and allows the update to be more strongly guided by advantage function. This monotonicity is consistent with the overall behavior of $\lambda$ in the CCI framework, as shown in Tab.~\ref{Tab-CCI-Lambda}. 

The detailed proof is provided in Appendix~\ref{Appendix-Prop-lambda-monotone}.

\begin{lemma}[Maximum Entropy Performance Difference]
\label{lem-soft-pdl}
Let $\pi$ and $\pi_\beta$ be two stationary policies. Consider the maximum-entropy performance $\mathcal{J} (\pi) = \mathbb{E}_{\tau \sim \pi } \left[ \sum_{t=0}^{\infty}\gamma^t\big(r(s_t,a_t)-\alpha\log\pi(a_t|s_t) \big) \right]$, then the maximum entropy performance difference satisfies
\begin{gather} 
\begin{aligned}
& \mathcal{J}(\pi)-\mathcal{J}(\pi_\beta)
= \\ & \frac{1}{1-\gamma}  \mathbb{E}_{s\sim d_\pi} \Big[
\mathbb{E}_{a\sim\pi(\cdot|s)}\big[\tilde A^{\pi_\beta}(s,a)\big] -\alpha \mathrm{KL}\big(\pi(\cdot|s) \| \pi_\beta(\cdot|s)\big)
\Big],
\label{Eq-soft-pdl-main}
\end{aligned}
\end{gather} 
where $\tilde A^{\pi_{\beta}} (s, a)$ is the advantage function under the shaped reward $\tilde r(s, a) = r(s,a)-\alpha\log\pi_\beta(a|s)$.
\end{lemma}

\begin{theorem}[Performance Lower Bound for the Optimal Policy]
\label{thm-optimal-policy-lower-bound}
Define $\epsilon_\beta(\pi) := \max_{s} \Big|\mathbb{E}_{a\sim\pi(\cdot|s)}\big[\tilde A^{\pi_\beta}(s,a)\big]\Big|$, $\kappa_{\beta}(\pi) := \max_{s} \mathrm{KL}\big(\pi(\cdot|s) \| \pi_\beta(\cdot|s)\big)$ and the average total variation distance under $\nu$ by $\Delta_{\mathrm{TV}} (\pi, \pi_\beta; \nu) := \mathbb{E}_{s\sim \nu} \big[D_{\mathrm{TV}}(\pi(\cdot|s), \pi_\beta(\cdot|s))\big]$. Assume $|\log \pi_\beta(a|s)| \le B$ for all $(s, a)$. For a given $\lambda\ge 0$ and $\alpha>0$, $\pi^*_{\lambda}$ satisfies the following maximum-entropy performance lower bound:
\begin{gather}
\begin{aligned}
\mathcal{J}(& \pi^*_\lambda) \ge \mathcal{J}(\pi_\beta) - \frac{2 B |\lambda - \alpha|}{1-\gamma} \Delta_{\mathrm{TV}} (\pi^*_\lambda,\pi_\beta; d_{\pi_\beta}) \\ & - \frac{4 \alpha B + 2 \gamma \Big(\epsilon_\beta (\pi^*_\lambda) + \alpha \kappa_{\beta} (\pi^*_\lambda)\Big)} {(1- \gamma)^2} 
\Delta_{\mathrm{TV}} (\pi^*_\lambda,\pi_\beta; d_{\pi_\beta}).
\end{aligned}
\label{Eq-Opt-Policy-LB}
\end{gather}
\end{theorem}

\begin{theorem}[Performance Lower Bound for a Parametric Policy]
\label{thm-para-policy-lower-bound}
Define the suboptimality gap of a parametric policy $\pi_\theta$ by
\begin{equation}
\delta^{\mathrm{sub}}_\theta(\lambda)
:=
F_\lambda(\pi^*_\lambda)-F_\lambda(\pi_\theta)
\;\ge\; 0,
\label{eq:subopt_gap_def}
\end{equation}
where $F_\lambda(\pi) = \mathbb{E}_{s\sim d_{\pi_\beta}} \Big[
\mathbb{E}_{a\sim\pi(\cdot|s)} \tilde A^{\pi_\beta}(s,a)
- \alpha \mathrm{KL} \big( \pi(\cdot|s) \| \pi_\beta(\cdot|s)\big)
+ (\lambda - \alpha) \mathbb{E}_{a\sim\pi(\cdot|s)} \log \pi_\beta(a|s) \Big]$. 

Then $\pi_\theta$ satisfies
\begin{gather}
\small
\begin{aligned}
\mathcal{J}(\pi_\theta)
& \ge\;
\mathcal{J}(\pi_\beta) -\frac{2B|\lambda - \alpha|}{1-\gamma}\,
\Delta_{\mathrm{TV}}\!\big(\pi_\theta,\pi_\beta; d_{\pi_\beta}\big)
-\frac{\delta^{\mathrm{sub}}_\theta(\lambda)}{1-\gamma}\\
&- \frac{4 \alpha B + 2 \gamma \Big(\epsilon_\beta (\pi_\theta) + \alpha \kappa_{\beta} (\pi_\theta)\Big)} {(1- \gamma)^2} 
\Delta_{\mathrm{TV}} (\pi_\theta,\pi_\beta; d_{\pi_\beta}).
\end{aligned}
\label{Eq-Para-Policy-LB}
\end{gather}
\end{theorem}

The proofs of Theorems \ref{thm-optimal-policy-lower-bound} and \ref{thm-para-policy-lower-bound} suggest a decomposition of the lower bound into two components. The first-order term $\mathcal{O}\!\left(\frac{1}{1-\gamma}\right)$ evaluates policy improvement under the fixed state distribution $d_{\pi_\beta}$.
The second-order term $\mathcal{O}\!\left(\frac{1}{(1-\gamma)^2}\right)$ captures distribution shift, arising from the mismatch between $d_\pi$ and $d_{\pi_\beta}$.

Compared to the KL-density constraint in standard RL, whose performance lower bound is typically $\mathcal{J}^{\mathrm{KL}}(\pi^*) \ge \mathcal{J}(\pi_\beta) - \mathcal{O}(\frac{1}{(1-\gamma)^2})$ \cite{MCQ}, our bound additionally exposes an explicit first-order term and makes the role of $\lambda$ transparent: the factor $|\lambda-\alpha|$ quantifies how the $\lambda$ modulates the worst-case performance. 

The detailed proof is provided in Appendix~\ref{Appendix-Per-Lower-Bound}.

\section{Experiments}

\subsection{Comparisons on D4RL Benchmarks}
\label{Subsec-D4RL}

We evaluate ACPO on D4RL \cite{D4RL} and compare it with a diverse set of strong offline RL baselines.

\textbf{Tasks.} We conduct experiments in Gym-MuJoCo locomotion domains and more challenging AntMaze and Kitchen tasks. The main challenge of the latter is to learn policies for long-horizon planning from datasets that do not contain optimal trajectories, which require ``stitching'' ability.

\textbf{Baselines.} We select competitive offline RL baselines, including BCQ \cite{BCQ}, IQL \cite{IQL}, CQL\cite{CQL}, TD3+BC \cite{TD3+BC}, SPOT \cite{SPOT}, SWG \cite{SWG}, EQL \cite{EQL}, DTQL \cite{DTQL}. 

\textbf{Results.} The whole results are shown in Tab.~\ref{Tab-D4RL}. Overall, ACPO demonstrates strong robustness across domains, achieving the best performance on Gym-MuJoCo and Kitchen tasks. On AntMaze, ACPO is slightly behind SWG and DTQL. A reason is that AntMaze requires trajectory-level planning and long-horizon credit assignment under sparse rewards, where diffusion-based methods (SWG/DTQL) are particularly effective. While these diffusion models achieve strong performance, their computational efficiency is also an important consideration.

\begin{table*}[!t]
\caption{Averaged normalized scores on Gym-MuJoCo-v2, Antmaze-v2, and Kitchen-v0 datasets over five seeds. Scores are reported as mean $\pm$ standard deviation. \textbf{Bold} indicates the best performance in each task.}
\label{Tab-D4RL}
\centering
\begin{tabular}{l||rrrrrrrr|r} 
\toprule
Dataset                 & BCQ   & IQL   & CQL   & TD3+BC & SPOT  & SWG   & EQL   & DTQL    & ACPO        \\
\hline
halfcheetah-med         & 47.0  & 47.4  & 44.0  & 48.3   & \textbf{58.4}  & 46.9  & 47.2  & 57.9   & 52.5 $\pm$ 0.76   \\
hopper-med              & 56.7  & 66.3  & 58.5  & 59.3   & 86.0  & 68.6  & 74.6  & \textbf{99.6}   & 97.9 $\pm$ 2.14  \\
walker2d-med            & 72.6  & 78.3  & 72.5  & 83.7   & 86.4  & 74.0  & 87.0  & \textbf{89.4}   & 88.0 $\pm$ 1.1   \\
halfcheetah-med-rep     & 40.4  & 44.2  & 45.5  & 44.6   & \textbf{52.2}  & 42.5  & 44.5  & 50.9   & 47.5 $\pm$ 0.54   \\
hopper-med-rep          & 53.3  & 94.7  & 95.0  & 60.9   & 100.2 & 73.9  & 98.1  & 100.0  & \textbf{100.6} $\pm$ 0.54  \\
walker2d-med-rep        & 52.1  & 73.9  & 77.2  & 81.8   & 91.6  & 60.6  & 76.6  & 88.5   & \textbf{93.6} $\pm$ 1.86   \\
halfcheetah-med-exp     & 89.1  & 86.7  & 90.7  & 90.7   & 86.9  & 92.8  & 90.6  & 92.7   & \textbf{97.2} $\pm$ 1.24   \\
hopper-med-exp          & 81.8  & 91.5  & 105.4 & 98.0   & 99.3  & 109.9 & 105.5 & 109.3  & \textbf{112.3} $\pm$ 0.72 \\
walker2d-med-exp        & 109.0 & 109.6 & 109.6 & 110.1  & 112.0 & 110.7 & 110.2 & 110.0  & \textbf{113.1} $\pm$ 0.22  \\
\hline
Locomotion Total        & 602   & 692.6 & 698.4 & 677.4  & 773   & 679.9 & 734.3 & 798.3  & \textbf{802.7} $\pm$ 9.12 \\
\hline
\hline
antmaze-u               & 78.9  & 85.5  & 84.8  & 78.6   & \textbf{93.5}  & 92.4  & 93.2  & 92.6   & 90.0 $\pm$ 3.0 \\
antmaze-u-d             & 55.0  & 66.7  & 43.4  & 71.4   & 40.7  & 70.0  & 65.4  & 74.4   & \textbf{86.0} $\pm$ 3.4 \\
antmaze-m-p             & 0     & 72.2  & 65.2  & 10.6   & 74.7  & \textbf{89.6}  & 77.5  & 76.0   & 76.0 $\pm$ 4.24 \\
antmaze-m-d             & 0     & 71.0  & 54.0  & 3.0    & 79.1  & \textbf{88.2}  & 70.0  & 80.6   & 76.0 $\pm$ 4.24 \\
antmaze-l-p             & 6.7   & 39.6  & 38.4  & 0.2    & 35.3  & 55.4  & 45.6  & \textbf{59.2}   & 56.0 $\pm$ 4.94 \\
antmaze-l-d             & 2.2   & 47.5  & 31.6  & 0.0    & 36.3  & 55.2  & 42.5  & \textbf{62.0}   & 48.0 $\pm$ 4.98 \\
\hline
Antmaze Total           & 142.8 & 382.5 & 317.4 & 163.8  & 359.6 & \textbf{450.8} & 394.2 & 444.8  & 432.0 $\pm$ 24.8 \\
\hline
\hline
kitchen-m               & 10.6  & 52.8  & 51.0  & 0.8    & 0.0   & 61.0   & 55.6  & 60.2   & \textbf{72.5} $\pm$ 5.5 \\
kitchen-p               & 18.9  & 46.1  & 49.8  & 0.0    & 0.0   & 69.0   & \textbf{74.5}  & 74.4   & 64.0 $\pm$ 15.52 \\
\hline
Kitchen Total           & 29.5  & 98.9  & 100.8 & 0.8    & 0.0   & 130.0  & 130.1 & 134.6  & \textbf{136.5} $\pm$ 21.02 \\
\bottomrule
\end{tabular}
\end{table*}

\subsection{Comparisons on NeoRL2 Benchmarks}
\label{Subsec-NeoRL2}

We also evaluate ACPO on NeoRL2 benchmarks.

\textbf{Tasks.} NeoRL2 (Near Real-World Offline RL Benchmark) \cite{NeoRL2} is a recent benchmark designed to reflect practical constraints in offline RL. It includes domains motivated by industrial control,  nuclear fusion, and healthcare, featuring pronounced time-delay effects and external disturbances. Moreover, practical data acquisition limits the dataset size and coverage, which places higher demands on the robustness and adaptiveness of offline RL algorithms.

\textbf{Baselines.} We compare against representative baselines commonly reported on NeoRL2, including BC, CQL \cite{CQL}, EDAC \cite{EDAC}, MCQ \cite{MCQ}, TD3+BC \cite{TD3+BC}, RAMBO \cite{RAMBO}, and MOBILE \cite{MOBILE} \footnote{We focus on the strongest baselines used in NeoRL2 and omit MOPO \cite{MOPO} and COMBO \cite{COMBO} for brevity.}.

\textbf{Results.} As shown in Tab.~\ref{Tab-NeoRL2}, ACPO achieves clear improvements on NeoRL2. It attains the best performance on $4$ out of $7$ datasets and yields the highest total score across tasks, substantially outperforming the strongest baseline TD3+BC.

\begin{table*}[!t]
\caption{Averaged normalized scores on NeoRL2 benchmark over three seeds. Scores are reported as mean $\pm$ standard deviation. \textbf{Bold} indicates the best performance in each task.}
\label{Tab-NeoRL2}
\centering
\resizebox{1.0\linewidth}{!}{
\begin{tabular}{l||rrrrrrrr|r}
\toprule
Dataset               & Data   & BC   & CQL   & EDAC  & MCQ   & TD3+BC & RAMBO   & MOBILE & ACPO \\
\hline
Pipeline              & 69.3  & 68.6  & 81.1  & 72.9  & 49.7  & 82.0   & 24.1    & 65.5   & \textbf{95.2} $\pm$ 4.62\\
Simglucose            & 73.9  & 75.1  & 11.0  & 8.1   & 29.6  & 74.2   & 10.8    & 9.3    & \textbf{94.9} $\pm$ 12.73\\
RocketRecovery        & 75.3  & 72.8  & 74.3  & 65.7  & 76.5  & 79.7   & -44.2   & 43.7   & \textbf{84.6} $\pm$ 10.23\\
RandomFrictionHopper  & 28.7  & 28.0  & 33.0  & 34.7  & 31.7  & 29.5   & 29.6    & \textbf{35.1}   & 33.7 $\pm$ 0.37\\
DMSD                  & 56.6  & 65.1  & 70.2  & \textbf{78.7}  & 77.8  & 60.0   & 76.2    & 64.4   & 74.9 $\pm$ 5.67\\
Fusion                & 48.8  & 55.2  & 55.9  & 58.0  & 49.7  & 54.6   & 59.6    & 5.0    & \textbf{61.8} $\pm$ 2.07\\
SafetyHalfCheetah     & \textbf{73.6}  & 70.2  & 71.2  & 53.1  & 54.7  & 68.6   & -422.4  & 8.7    & 67.1 $\pm$ 6.90\\
\hline
Total                 & 426.2 & 435.0 & 396.7 & 371.2 & 369.7 & 448.6 & -266.3   & 231.7  & \textbf{512.2} $\pm$ 42.59\\
\bottomrule
\end{tabular}}
\end{table*}

\subsection{Special-Case Comparisons and $\lambda$ Dynamics}
\label{Subsec-Threecase}

Under the CCI framework, several representative policy constraints can be recovered by fixing the interpolation parameter $\lambda$. To isolate the effect of $\lambda$, we perform special-case comparisons by disabling the dual update in ACPO and training with a fixed $\lambda$ throughout. These special cases match wBC, AWAC \cite{AWAC} and InAC \cite{InAC} at the level of the policy improvement. In contrast, the policy evaluation components may differ across methods; for a controlled comparison, we keep the policy evaluation fixed and use the same maximum-entropy critic updates Eq.~\eqref{Eq-Loss-V} and~\eqref{Eq-Loss-Q} as ACPO in all variants.

\textbf{Settings.} For Gym-MuJoCo tasks, we set $\alpha=0.1$. For NeoRL2 tasks, we set $\alpha=0.5$. We instantiate three special cases by choosing: (i) practical wBC: $\lambda = 100$ to approximate the wBC regime, (ii) AWAC: $\lambda = \alpha$, and (iii) InAC: $\lambda = 0$.
\begin{figure*}[htbp] 
\centerline{\includegraphics[width=1.0\textwidth]{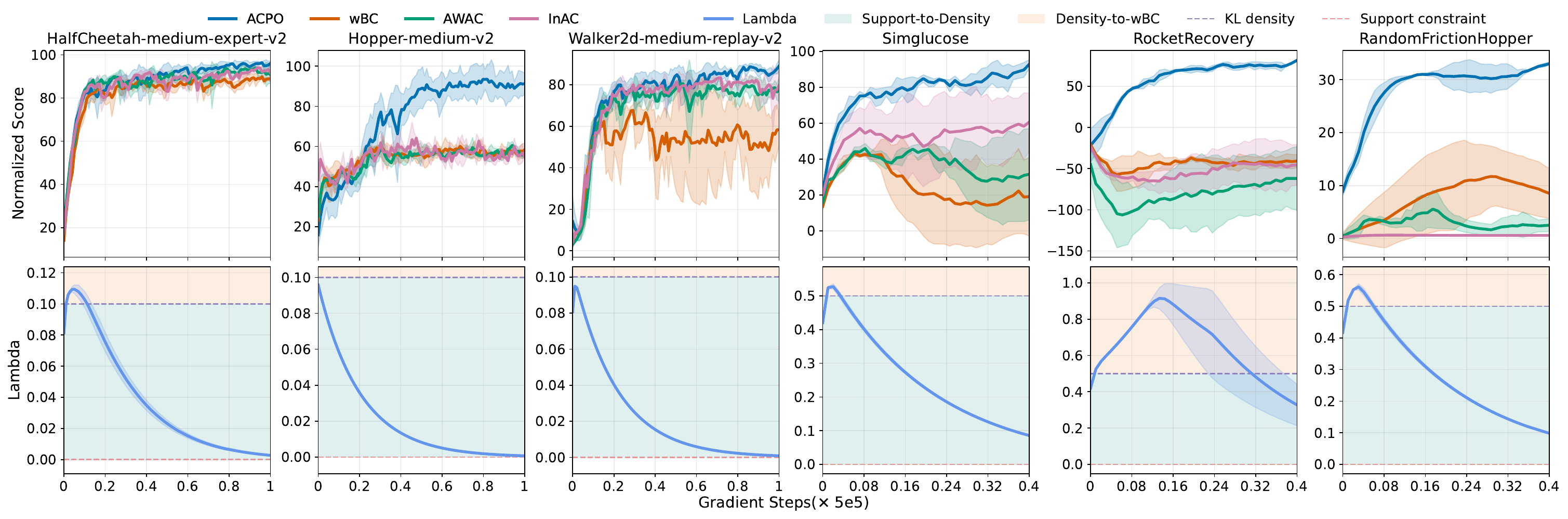}}
\caption{Special-case comparisons and the evolution of $\lambda$ learned by ACPO during training.}
\label{Fig-ScoreCompare_and_LambdaAnnotated_2x6}
\end{figure*}

Fig.~\ref{Fig-ScoreCompare_and_LambdaAnnotated_2x6} reports training curves of ACPO against these three baselines across six datasets, together with the adaptive evolution of $\lambda$ during training. From Fig.~\ref{Fig-ScoreCompare_and_LambdaAnnotated_2x6}, we draw the following observations:

\textbf{Overall Superiority.} ACPO consistently achieves the strongest performance across the six tasks, with particularly clear gains on the more challenging NeoRL2 datasets where coverage is limited. This highlights the benefit of automatically adapting the effective constraint regime to the dataset.

\textbf{$\lambda$ Dynamics.} We observe that $\lambda$ is often non-monotonic, typically rising in early training and gradually decaying afterwards. Early in training, critic estimates are less calibrated and actor updates are more prone to exploiting OOD actions. The dual update reacts by increasing $\lambda$, strengthening the constraint and biasing updates toward higher-density behavior actions to stabilize learning. As training progresses and value estimates become better calibrated, the OOD risk diminishes; $\lambda$ decreases accordingly, relaxing the constraint and enabling more effective policy improvement.

\subsection{Behavior Density Estimator Ablation Study}
\label{Subsec-Behavior}
Since ACPO explicitly depends on $\log \pi_\beta(a|s)$, the choice of the behavior density estimator can directly affect both the constraint signal and the policy improvement weights. 

Beyond a Gaussian density estimator for $\pi_\beta$, we additionally study the conditional variational auto-encoder (CVAE) model \cite{BCQ, ELAPSE} to approximate $\pi_\beta$. Moreover, motivated by prior work \cite{SPOT}, we also adapt the IS corrected method to reduce estimation bias in expectations of $\pi_\beta$. We compare Gaussian vs. CVAE behavior models across $9$ Gym-MuJoCo datasets, with results reported in Fig.~\ref{Fig-Gaussian_minus_CVAE}.
\begin{figure}[!t] 
\centerline{\includegraphics[width=0.48\textwidth]{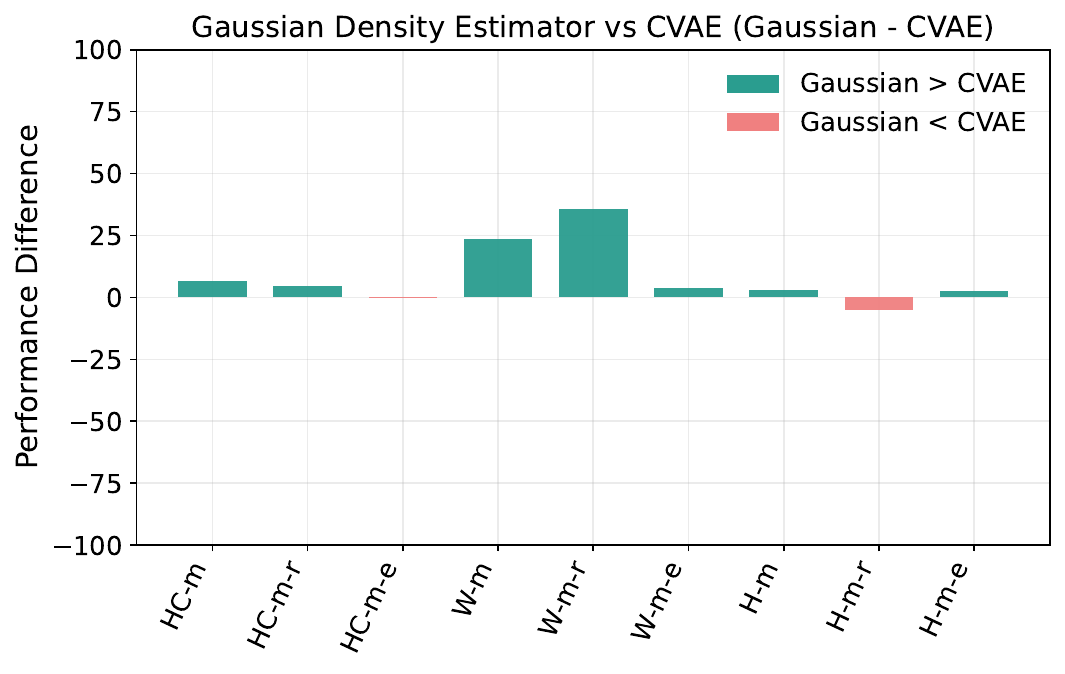}}
\caption{Normalized score differences on Gym-MuJoCo datasets.}
\label{Fig-Gaussian_minus_CVAE}
\end{figure}

As shown in Fig.~\ref{Fig-Gaussian_minus_CVAE}, the Gaussian and CVAE behavior models yield comparable performance on most datasets, except for \texttt{walker2d-medium-v2} and \texttt{walker2d-medium-replay-v2}, where the performance gap becomes pronounced. To diagnose this discrepancy, we visualize the distributions of $\log \pi_\beta(a|s)$ induced by the Gaussian and CVAE estimators on (i) in-dataset actions $a_{\mathrm{in}}$ and (ii) perturbed OOD actions $a_{\mathrm{ood}}$ constructed as $a_{ood} = a_{in} + \xi z$, where $\xi$ controls the perturbation scale and $z \sim \mathcal{N}(0, I)$. The resulting distributions are shown in Fig.~\ref{Fig-Gaussian_CVAE_log},
\begin{figure}[!htbp] 
\centerline{\includegraphics[width=0.49\textwidth]{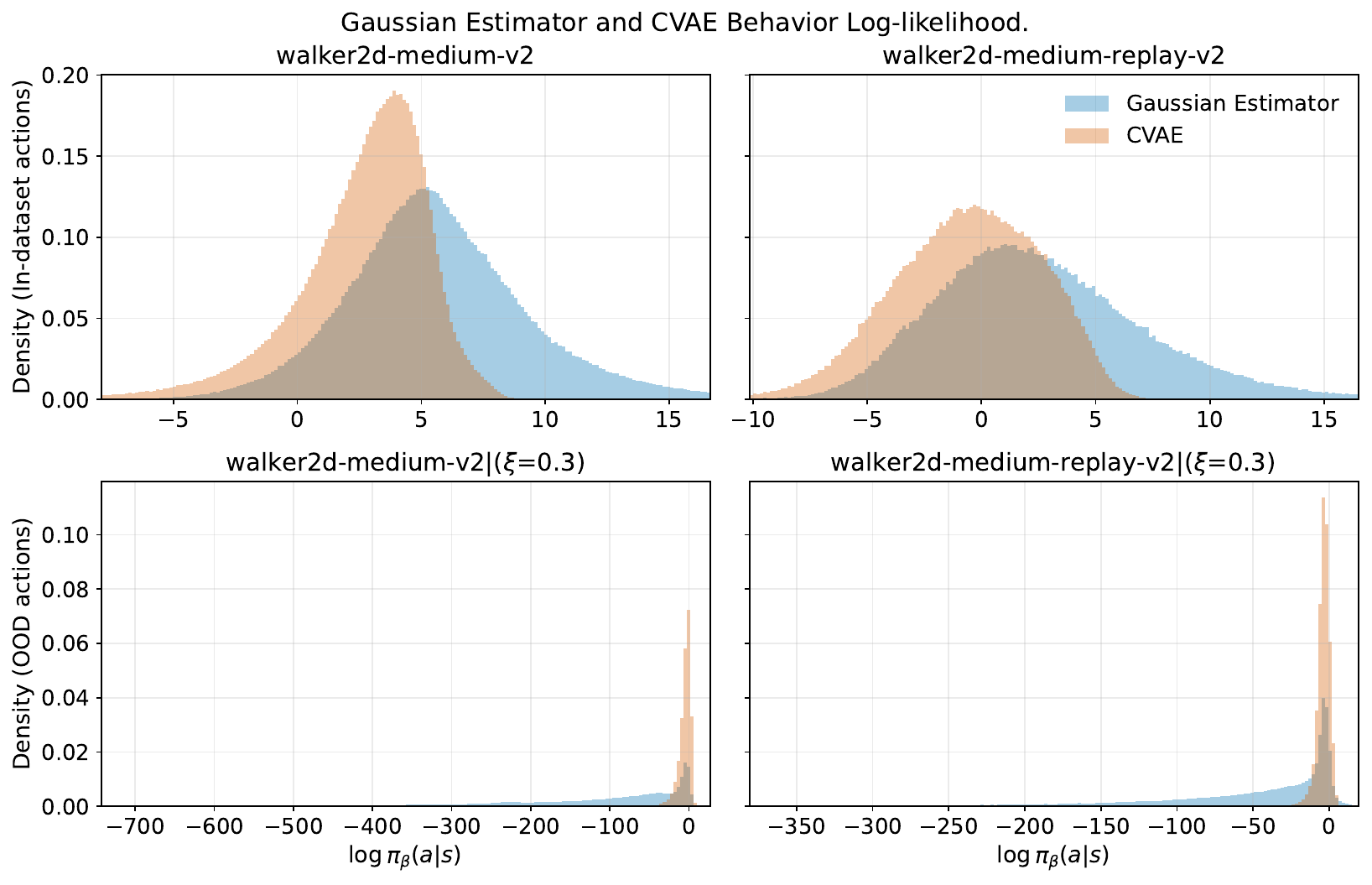}}
\caption{Distributions of $\log \pi_\beta(a|s)$ under Gaussian and CVAE behavior estimators on in-dataset actions and OOD actions.}
\label{Fig-Gaussian_CVAE_log}
\end{figure}

Fig.~\ref{Fig-Gaussian_CVAE_log} explains the reason for weaker performance of the CVAE estimator on \texttt{walker2d-medium-v2} and \texttt{walker2d-medium-replay-v2}. On in-dataset actions, the CVAE induces a noticeably sharper $\log \pi_\beta(a|s)$ distribution than the Gaussian model, suggesting a more concentrated density estimate. More importantly, on perturbed OOD actions, the CVAE does not consistently assign sufficiently low likelihood compared with the Gaussian estimator: the distribution of $\log \pi_\beta(a_{\mathrm{ood}}|s)$ is not shifted left as much as expected under clearly OOD perturbations. This overestimation of behavior likelihood for OOD actions can weaken the constraint signal in ACPO, increase susceptibility to extrapolation error, and ultimately degrade performance.

Training details for all experiments are provided in Appendix~\ref{Appendix-Experiment}.

\section{Conclusion}
We proposed CCI, a unified optimization framework that bridges weighted BC, KL-based density regularization, and support constraints as special cases along a common constraint spectrum. CCI exposes an explicit interpolation parameter that enables smooth transitions and principled combinations across constraint regimes. Building on CCI, we developed ACPO, a practical primal--dual algorithm that automatically adapts this constraint dial via a Lagrangian dual update. We further established a maximum-entropy performance difference lemma and derived performance lower bounds for both the closed-form optimal policy and its parametric projection. Experiments on D4RL and NeoRL2 show that ACPO achieves robust gains across diverse domains, often matching or exceeding strong baselines overall. As future work, extending ACPO to trajectory-level settings for long-horizon credit assignment and sparse rewards tasks is critical. Another promising direction is to incorporate dataset filtering to improve the reliability of behavior policy estimation.

\section{Impact Statements}
This paper presents work whose goal is to advance the field of machine learning. There are many potential societal consequences of our work, none of which we feel must be specifically highlighted here.


\bibliography{example_paper}

\begin{thebibliography}{58}
\providecommand{\natexlab}[1]{#1}
\providecommand{\url}[1]{\texttt{#1}}
\expandafter\ifx\csname urlstyle\endcsname\relax
  \providecommand{\doi}[1]{doi: #1}\else
  \providecommand{\doi}{doi: \begingroup \urlstyle{rm}\Url}\fi

\bibitem[Achiam et~al.(2017)Achiam, Held, Tamar, and Abbeel]{CPO}
Achiam, J., Held, D., Tamar, A., and Abbeel, P.
\newblock Constrained policy optimization.
\newblock In \emph{Proceedings of the 34th International Conference on Machine Learning}, volume~70 of \emph{Proceedings of Machine Learning Research}, pp.\  22--31. PMLR, 06--11 Aug 2017.

\bibitem[An et~al.(2021)An, Moon, Kim, and Song]{EDAC}
An, G., Moon, S., Kim, J.-H., and Song, H.~O.
\newblock Uncertainty-based offline reinforcement learning with diversified q-ensemble.
\newblock \emph{Advances in neural information processing systems}, 34:\penalty0 7436--7447, 2021.

\bibitem[Asadi \& Littman(2017)Asadi and Littman]{MERL}
Asadi, K. and Littman, M.~L.
\newblock An alternative softmax operator for reinforcement learning.
\newblock In \emph{Proceedings of the 34th International Conference on Machine Learning}, volume~70 of \emph{Proceedings of Machine Learning Research}, pp.\  243--252. PMLR, 06--11 Aug 2017.

\bibitem[Asadulaev et~al.(2024)Asadulaev, Korst, Korotin, Egiazarian, Filchenkov, and Burnaev]{OT}
Asadulaev, A., Korst, R., Korotin, A., Egiazarian, V., Filchenkov, A., and Burnaev, E.
\newblock Rethinking optimal transport in offline reinforcement learning.
\newblock In \emph{Advances in Neural Information Processing Systems}, volume~37, pp.\  123592--123607. Curran Associates, Inc., 2024.
\newblock \doi{10.52202/079017-3929}.

\bibitem[Bai et~al.(2022)Bai, Wang, Yang, Deng, Garg, Liu, and Wang]{PBRL}
Bai, C., Wang, L., Yang, Z., Deng, Z.-H., Garg, A., Liu, P., and Wang, Z.
\newblock Pessimistic bootstrapping for uncertainty-driven offline reinforcement learning.
\newblock In \emph{International Conference on Learning Representations}, 2022.

\bibitem[Chen et~al.(2024)Chen, Wang, and Zhou]{DTQL}
Chen, T., Wang, Z., and Zhou, M.
\newblock Diffusion policies creating a trust region for offline reinforcement learning.
\newblock In \emph{Advances in Neural Information Processing Systems}, volume~37, pp.\  50098--50125. Curran Associates, Inc., 2024.
\newblock \doi{10.52202/079017-1585}.

\bibitem[Chen et~al.(2025)Chen, Cai, Wu, and Zhang]{ACTIVE}
Chen, T., Cai, R., Wu, F., and Zhang, X.
\newblock {ACTIVE}: Offline reinforcement learning via adaptive imitation and in-sample \$v\$-ensemble.
\newblock In \emph{The Thirteenth International Conference on Learning Representations}, 2025.
\newblock URL \url{https://openreview.net/forum?id=qiluFujVc8}.

\bibitem[Chen et~al.(2020)Chen, Zhou, Wang, Wang, Wu, and Ross]{BAIL}
Chen, X., Zhou, Z., Wang, Z., Wang, C., Wu, Y., and Ross, K.
\newblock Bail: Best-action imitation learning for batch deep reinforcement learning.
\newblock In Larochelle, H., Ranzato, M., Hadsell, R., Balcan, M., and Lin, H. (eds.), \emph{Advances in Neural Information Processing Systems}, volume~33, pp.\  18353--18363. Curran Associates, Inc., 2020.

\bibitem[Chen et~al.(2022)Chen, Ghadirzadeh, Yu, Wang, Gao, Li, Bin, Finn, and Zhang]{LAPO}
Chen, X., Ghadirzadeh, A., Yu, T., Wang, J., Gao, Y., Li, W., Bin, L., Finn, C., and Zhang, C.
\newblock {LAPO}: Latent-variable advantage-weighted policy optimization for offline reinforcement learning.
\newblock In Oh, A.~H., Agarwal, A., Belgrave, D., and Cho, K. (eds.), \emph{Advances in Neural Information Processing Systems}, 2022.

\bibitem[Florence et~al.(2022)Florence, Lynch, Zeng, Ramirez, Wahid, Downs, Wong, Lee, Mordatch, and Tompson]{IBC}
Florence, P., Lynch, C., Zeng, A., Ramirez, O.~A., Wahid, A., Downs, L., Wong, A., Lee, J., Mordatch, I., and Tompson, J.
\newblock Implicit behavioral cloning.
\newblock In \emph{Proceedings of the 5th Conference on Robot Learning}, volume 164 of \emph{Proceedings of Machine Learning Research}, pp.\  158--168. PMLR, 08--11 Nov 2022.

\bibitem[Fu et~al.(2021)Fu, Kumar, Nachum, Tucker, and Levine]{D4RL}
Fu, J., Kumar, A., Nachum, O., Tucker, G., and Levine, S.
\newblock D4rl: Datasets for deep data-driven reinforcement learning, 2021.
\newblock URL \url{https://arxiv.org/abs/2004.07219}.

\bibitem[Fujimoto \& Gu(2021)Fujimoto and Gu]{TD3+BC}
Fujimoto, S. and Gu, S.~S.
\newblock A minimalist approach to offline reinforcement learning.
\newblock In \emph{Advances in Neural Information Processing Systems}, volume~34, pp.\  20132--20145. Curran Associates, Inc., 2021.

\bibitem[Fujimoto et~al.(2018)Fujimoto, van Hoof, and Meger]{ClipDoubleQ}
Fujimoto, S., van Hoof, H., and Meger, D.
\newblock Addressing function approximation error in actor-critic methods.
\newblock In \emph{Proceedings of the 35th International Conference on Machine Learning}, volume~80 of \emph{Proceedings of Machine Learning Research}, pp.\  1587--1596. PMLR, 10--15 Jul 2018.

\bibitem[Fujimoto et~al.(2019)Fujimoto, Meger, and Precup]{BCQ}
Fujimoto, S., Meger, D., and Precup, D.
\newblock Off-policy deep reinforcement learning without exploration.
\newblock In \emph{Proceedings of the 36th International Conference on Machine Learning}, volume~97 of \emph{Proceedings of Machine Learning Research}, pp.\  2052--2062. PMLR, 09--15 Jun 2019.

\bibitem[Gao et~al.(2025)Gao, Tu, Qin, Sun, Chen, and Yu]{NeoRL2}
Gao, S., Tu, Z., Qin, R.-J., Sun, Y.-H., Chen, X.-H., and Yu, Y.
\newblock Neorl-2: Near real-world benchmarks for offline reinforcement learning with extended realistic scenarios, 2025.

\bibitem[Ghasemipour et~al.(2021)Ghasemipour, Schuurmans, and Gu]{EMaQ}
Ghasemipour, S. K.~S., Schuurmans, D., and Gu, S.~S.
\newblock Emaq: Expected-max q-learning operator for simple yet effective offline and online rl.
\newblock In \emph{Proceedings of the 38th International Conference on Machine Learning}, volume 139 of \emph{Proceedings of Machine Learning Research}, pp.\  3682--3691. PMLR, 18--24 Jul 2021.

\bibitem[Haarnoja et~al.(2017)Haarnoja, Tang, Abbeel, and Levine]{SQL}
Haarnoja, T., Tang, H., Abbeel, P., and Levine, S.
\newblock Reinforcement learning with deep energy-based policies.
\newblock In \emph{Proceedings of the 34th International Conference on Machine Learning}, volume~70 of \emph{Proceedings of Machine Learning Research}, pp.\  1352--1361. PMLR, 06--11 Aug 2017.

\bibitem[Haarnoja et~al.(2018)Haarnoja, Zhou, Abbeel, and Levine]{SAC}
Haarnoja, T., Zhou, A., Abbeel, P., and Levine, S.
\newblock Soft actor-critic: Off-policy maximum entropy deep reinforcement learning with a stochastic actor.
\newblock In \emph{Proceedings of the 35th International Conference on Machine Learning}, volume~80 of \emph{Proceedings of Machine Learning Research}, pp.\  1861--1870. PMLR, 10--15 Jul 2018.

\bibitem[Han et~al.(2024)Han, Afifi, Moungla, and Marot]{LeakyPPO}
Han, X., Afifi, H., Moungla, H., and Marot, M.
\newblock Leaky ppo: A simple and efficient rl algorithm for autonomous vehicles.
\newblock In \emph{2024 International Joint Conference on Neural Networks (IJCNN)}, pp.\  1--7, 2024.
\newblock \doi{10.1109/IJCNN60899.2024.10650450}.

\bibitem[Han et~al.(2025{\natexlab{a}})Han, Afifi, and Marot]{CosIQL}
Han, X., Afifi, H., and Marot, M.
\newblock Cosine similarity based adaptive implicit q-learning for offline reinforcement learning.
\newblock In \emph{2025 IEEE Wireless Communications and Networking Conference (WCNC)}, pp.\  1--6, 2025{\natexlab{a}}.
\newblock \doi{10.1109/WCNC61545.2025.10978817}.

\bibitem[Han et~al.(2025{\natexlab{b}})Han, Afifi, and Marot]{ELAPSE}
Han, X., Afifi, H., and Marot, M.
\newblock Elapse: Expand latent action projection space for policy optimization in offline reinforcement learning.
\newblock \emph{Neurocomputing}, 631:\penalty0 129665, 2025{\natexlab{b}}.
\newblock ISSN 0925-2312.
\newblock \doi{https://doi.org/10.1016/j.neucom.2025.129665}.

\bibitem[Han et~al.(2025{\natexlab{c}})Han, Afifi, and Marot]{PIQL}
Han, X., Afifi, H., and Marot, M.
\newblock {PIQL}: Projective implicit q-learning with support constraint for offline reinforcement learning.
\newblock In \emph{The 25th International Conference on Autonomous Agents and Multi-Agent Systems}, 2025{\natexlab{c}}.

\bibitem[Han et~al.(2025{\natexlab{d}})Han, Afifi, Moungla, and Marot]{AEM}
Han, X., Afifi, H., Moungla, H., and Marot, M.
\newblock Attention ensemble mixture: a novel offline reinforcement learning algorithm for autonomous vehicles.
\newblock \emph{Applied Intelligence}, 55\penalty0 (6):\penalty0 1--14, 2025{\natexlab{d}}.

\bibitem[Kakade \& Langford(2002)Kakade and Langford]{PDT}
Kakade, S. and Langford, J.
\newblock Approximately optimal approximate reinforcement learning.
\newblock In \emph{Proceedings of the Nineteenth International Conference on Machine Learning}, ICML '02, pp.\  267–274, San Francisco, CA, USA, 2002. Morgan Kaufmann Publishers Inc.
\newblock ISBN 1558608737.

\bibitem[Kim et~al.(2025)Kim, Shin, Jung, Hong, Yoon, Sung, Lee, and Lim]{PARS}
Kim, J., Shin, Y., Jung, W., Hong, S., Yoon, D., Sung, Y., Lee, K., and Lim, W.
\newblock Penalizing infeasible actions and reward scaling in reinforcement learning with offline data.
\newblock In \emph{Forty-second International Conference on Machine Learning}, 2025.

\bibitem[Kostrikov et~al.(2021{\natexlab{a}})Kostrikov, Fergus, Tompson, and Nachum]{Fisher-BRC}
Kostrikov, I., Fergus, R., Tompson, J., and Nachum, O.
\newblock Offline reinforcement learning with fisher divergence critic regularization.
\newblock In \emph{Proceedings of the 38th International Conference on Machine Learning}, volume 139 of \emph{Proceedings of Machine Learning Research}, pp.\  5774--5783. PMLR, 18--24 Jul 2021{\natexlab{a}}.

\bibitem[Kostrikov et~al.(2021{\natexlab{b}})Kostrikov, Nair, and Levine]{IQL}
Kostrikov, I., Nair, A., and Levine, S.
\newblock Offline reinforcement learning with implicit q-learning.
\newblock In \emph{Deep RL Workshop NeurIPS 2021}, 2021{\natexlab{b}}.
\newblock URL \url{https://openreview.net/forum?id=EblVBDNalKu}.

\bibitem[Kumar et~al.(2019)Kumar, Fu, Soh, Tucker, and Levine]{BEAR}
Kumar, A., Fu, J., Soh, M., Tucker, G., and Levine, S.
\newblock Stabilizing off-policy q-learning via bootstrapping error reduction.
\newblock \emph{Advances in neural information processing systems}, 32, 2019.

\bibitem[Kumar et~al.(2020)Kumar, Zhou, Tucker, and Levine]{CQL}
Kumar, A., Zhou, A., Tucker, G., and Levine, S.
\newblock Conservative q-learning for offline reinforcement learning.
\newblock In \emph{Advances in Neural Information Processing Systems}, volume~33, pp.\  1179--1191. Curran Associates, Inc., 2020.

\bibitem[Lee et~al.(2021)Lee, Jeon, Lee, Pineau, and Kim]{OptiDICE}
Lee, J., Jeon, W., Lee, B., Pineau, J., and Kim, K.-E.
\newblock Optidice: Offline policy optimization via stationary distribution correction estimation.
\newblock In \emph{Proceedings of the 38th International Conference on Machine Learning}, volume 139 of \emph{Proceedings of Machine Learning Research}, pp.\  6120--6130. PMLR, 18--24 Jul 2021.

\bibitem[Levine et~al.(2020)Levine, Kumar, Tucker, and Fu]{ORLtutorial}
Levine, S., Kumar, A., Tucker, G., and Fu, J.
\newblock Offline reinforcement learning: Tutorial, review, and perspectives on open problems.
\newblock \emph{arXiv preprint arXiv:2005.01643}, 2020.

\bibitem[Liu et~al.(2024)Liu, Zhang, Li, Yang, Liu, and Ouyang]{APTQ}
Liu, J., Zhang, Y., Li, C., Yang, Y., Liu, Y., and Ouyang, W.
\newblock Adaptive pessimism via target q-value for offline reinforcement learning.
\newblock \emph{Neural Networks}, 180:\penalty0 106588, 2024.
\newblock ISSN 0893-6080.
\newblock \doi{https://doi.org/10.1016/j.neunet.2024.106588}.

\bibitem[Luo et~al.(2025)Luo, Xie, Wang, and Huang]{SSAR}
Luo, Q.-W., Xie, M.-K., Wang, Y.-W., and Huang, S.-J.
\newblock Learning to trust bellman updates: Selective state-adaptive regularization for offline {RL}.
\newblock In \emph{Forty-second International Conference on Machine Learning}, 2025.

\bibitem[Lyu et~al.(2022)Lyu, Ma, Li, and Lu]{MCQ}
Lyu, J., Ma, X., Li, X., and Lu, Z.
\newblock Mildly conservative q-learning for offline reinforcement learning.
\newblock In \emph{Advances in Neural Information Processing Systems}, volume~35, pp.\  1711--1724. Curran Associates, Inc., 2022.

\bibitem[Mao et~al.(2024)Mao, Xu, Zhan, Zhang, and Zhang]{DiffDICE}
Mao, L., Xu, H., Zhan, X., Zhang, W., and Zhang, A.
\newblock Diffusion-dice: In-sample diffusion guidance for offline reinforcement learning.
\newblock In Globerson, A., Mackey, L., Belgrave, D., Fan, A., Paquet, U., Tomczak, J., and Zhang, C. (eds.), \emph{Advances in Neural Information Processing Systems}, volume~37, pp.\  98806--98834. Curran Associates, Inc., 2024.
\newblock \doi{10.52202/079017-3136}.

\bibitem[Mao et~al.(2023)Mao, Zhang, Chen, Xu, and Ji]{STR}
Mao, Y., Zhang, H., Chen, C., Xu, Y., and Ji, X.
\newblock Supported trust region optimization for offline reinforcement learning.
\newblock In \emph{International Conference on Machine Learning}, pp.\  23829--23851. PMLR, 2023.

\bibitem[Nachum et~al.(2019)Nachum, Chow, Dai, and Li]{DualDICE}
Nachum, O., Chow, Y., Dai, B., and Li, L.
\newblock Dualdice: Behavior-agnostic estimation of discounted stationary distribution corrections.
\newblock In \emph{Advances in Neural Information Processing Systems}, volume~32. Curran Associates, Inc., 2019.

\bibitem[Nair et~al.(2020)Nair, Gupta, Dalal, and Levine]{AWAC}
Nair, A., Gupta, A., Dalal, M., and Levine, S.
\newblock Awac: Accelerating online reinforcement learning with offline datasets.
\newblock \emph{arXiv preprint arXiv:2006.09359}, 2020.

\bibitem[Omura et~al.(2025)Omura, Mukuta, Ota, Osa, and Harada]{Wasserstein}
Omura, M., Mukuta, Y., Ota, K., Osa, T., and Harada, T.
\newblock Offline reinforcement learning with wasserstein regularization via optimal transport maps.
\newblock In \emph{Reinforcement Learning Conference}, 2025.
\newblock URL \url{https://openreview.net/forum?id=RxdQBYKjtE}.

\bibitem[Peng et~al.(2019)Peng, Kumar, Zhang, and Levine]{AWR}
Peng, X.~B., Kumar, A., Zhang, G., and Levine, S.
\newblock Advantage-weighted regression: Simple and scalable off-policy reinforcement learning.
\newblock \emph{arXiv preprint arXiv:1910.00177}, 2019.

\bibitem[Peng et~al.(2023)Peng, Liu, Chen, and Zhou]{CEN}
Peng, Z., Liu, Y., Chen, H., and Zhou, Z.
\newblock Conservative network for offline reinforcement learning.
\newblock \emph{Knowledge-Based Systems}, 282:\penalty0 111101, 2023.
\newblock ISSN 0950-7051.
\newblock \doi{https://doi.org/10.1016/j.knosys.2023.111101}.

\bibitem[Peters \& Schaal(2007)Peters and Schaal]{RWR}
Peters, J. and Schaal, S.
\newblock Reinforcement learning by reward-weighted regression for operational space control.
\newblock In \emph{Proceedings of the 24th international conference on Machine learning}, pp.\  745--750, 2007.

\bibitem[Rigter et~al.(2022)Rigter, Lacerda, and Hawes]{RAMBO}
Rigter, M., Lacerda, B., and Hawes, N.
\newblock Rambo-rl: Robust adversarial model-based offline reinforcement learning.
\newblock In \emph{Advances in Neural Information Processing Systems}, volume~35, pp.\  16082--16097. Curran Associates, Inc., 2022.

\bibitem[Siegel et~al.(2020)Siegel, Springenberg, Berkenkamp, Abdolmaleki, Neunert, Lampe, Hafner, Heess, and Riedmiller]{ABM}
Siegel, N., Springenberg, J.~T., Berkenkamp, F., Abdolmaleki, A., Neunert, M., Lampe, T., Hafner, R., Heess, N., and Riedmiller, M.
\newblock Keep doing what worked: Behavior modelling priors for offline reinforcement learning.
\newblock In \emph{International Conference on Learning Representations}, 2020.

\bibitem[Sun et~al.(2023)Sun, Zhang, Jia, Lin, Ye, and Yu]{MOBILE}
Sun, Y., Zhang, J., Jia, C., Lin, H., Ye, J., and Yu, Y.
\newblock Model-{B}ellman inconsistency for model-based offline reinforcement learning.
\newblock In \emph{Proceedings of the 40th International Conference on Machine Learning}, volume 202 of \emph{Proceedings of Machine Learning Research}, pp.\  33177--33194. PMLR, 23--29 Jul 2023.

\bibitem[Sutton \& Barto(2018)Sutton and Barto]{SuttonRL}
Sutton, R.~S. and Barto, A.~G.
\newblock \emph{Reinforcement learning: An introduction}.
\newblock MIT press, 2018.

\bibitem[Tagle et~al.(2025)Tagle, del solar, and Tobar]{SWG}
Tagle, A., del solar, J.~R., and Tobar, F.
\newblock Diffusion self-weighted guidance for offline reinforcement learning.
\newblock \emph{Transactions on Machine Learning Research}, 2025.
\newblock ISSN 2835-8856.
\newblock URL \url{https://openreview.net/forum?id=jmXBnpmznv}.

\bibitem[Wang et~al.(2020)Wang, Novikov, Zolna, Merel, Springenberg, Reed, Shahriari, Siegel, Gulcehre, Heess, et~al.]{CRR}
Wang, Z., Novikov, A., Zolna, K., Merel, J.~S., Springenberg, J.~T., Reed, S.~E., Shahriari, B., Siegel, N., Gulcehre, C., Heess, N., et~al.
\newblock Critic regularized regression.
\newblock \emph{Advances in Neural Information Processing Systems}, 33:\penalty0 7768--7778, 2020.

\bibitem[Wu et~al.(2022)Wu, Wu, Qiu, Wang, and Long]{SPOT}
Wu, J., Wu, H., Qiu, Z., Wang, J., and Long, M.
\newblock Supported policy optimization for offline reinforcement learning.
\newblock \emph{Advances in Neural Information Processing Systems}, 35:\penalty0 31278--31291, 2022.

\bibitem[Wu et~al.(2025)Wu, Zhao, Xu, Che, Yin, Harold~Liu, Feng, and Tang]{ACLQL}
Wu, K., Zhao, Y., Xu, Z., Che, Z., Yin, C., Harold~Liu, C., Feng, F., and Tang, J.
\newblock Acl-ql: Adaptive conservative level in q-learning for offline reinforcement learning.
\newblock \emph{IEEE Transactions on Neural Networks and Learning Systems}, 36\penalty0 (6):\penalty0 11399--11413, 2025.
\newblock \doi{10.1109/TNNLS.2024.3497667}.

\bibitem[Wu et~al.(2020)Wu, Tucker, and Nachum]{BRAC}
Wu, Y., Tucker, G., and Nachum, O.
\newblock Behavior regularized offline reinforcement learning, 2020.
\newblock URL \url{https://openreview.net/forum?id=BJg9hTNKPH}.

\bibitem[Xiao et~al.(2019)Xiao, Huang, Mei, Schuurmans, and M\"{u}ller]{MEMC}
Xiao, C., Huang, R., Mei, J., Schuurmans, D., and M\"{u}ller, M.
\newblock Maximum entropy monte-carlo planning.
\newblock In \emph{Advances in Neural Information Processing Systems}, volume~32. Curran Associates, Inc., 2019.

\bibitem[Xiao et~al.(2023)Xiao, Wang, Pan, White, and White]{InAC}
Xiao, C., Wang, H., Pan, Y., White, A., and White, M.
\newblock The in-sample softmax for offline reinforcement learning.
\newblock In \emph{The Eleventh International Conference on Learning Representations}, 2023.

\bibitem[Xu et~al.(2022)Xu, Zhan, Yin, and Qin]{DWBC}
Xu, H., Zhan, X., Yin, H., and Qin, H.
\newblock Discriminator-weighted offline imitation learning from suboptimal demonstrations.
\newblock In \emph{Proceedings of the 39th International Conference on Machine Learning}, volume 162 of \emph{Proceedings of Machine Learning Research}, pp.\  24725--24742. PMLR, 17--23 Jul 2022.

\bibitem[Xu et~al.(2023)Xu, Jiang, Li, Yang, Wang, Chan, and Zhan]{EQL}
Xu, H., Jiang, L., Li, J., Yang, Z., Wang, Z., Chan, V. W.~K., and Zhan, X.
\newblock Offline {RL} with no {OOD} actions: In-sample learning via implicit value regularization.
\newblock In \emph{The Eleventh International Conference on Learning Representations}, 2023.
\newblock URL \url{https://openreview.net/forum?id=ueYYgo2pSSU}.

\bibitem[Yu et~al.(2020)Yu, Thomas, Yu, Ermon, Zou, Levine, Finn, and Ma]{MOPO}
Yu, T., Thomas, G., Yu, L., Ermon, S., Zou, J.~Y., Levine, S., Finn, C., and Ma, T.
\newblock Mopo: Model-based offline policy optimization.
\newblock In \emph{Advances in Neural Information Processing Systems}, volume~33, pp.\  14129--14142. Curran Associates, Inc., 2020.

\bibitem[Yu et~al.(2021)Yu, Kumar, Rafailov, Rajeswaran, Levine, and Finn]{COMBO}
Yu, T., Kumar, A., Rafailov, R., Rajeswaran, A., Levine, S., and Finn, C.
\newblock Combo: Conservative offline model-based policy optimization.
\newblock In \emph{Advances in Neural Information Processing Systems}, volume~34, pp.\  28954--28967. Curran Associates, Inc., 2021.

\bibitem[Zhou et~al.(2022)Zhou, Li, and Qu]{ABR}
Zhou, Y., Li, X., and Qu, Q.
\newblock Offline reinforcement learning with adaptive behavior regularization, 2022.

\end{thebibliography}
\bibliographystyle{icml2026}

\newpage
\appendix
\onecolumn

\section{Continuous Constraint Interpolation Framework}
\label{Appendix-CCI}

We start from the following constrained maximum-entropy policy optimization problem,
\begin{gather} 
\begin{aligned}
\max_{\pi} \ & \mathbb{E}_{s\sim \mathcal{D}, a\sim \pi(\cdot|s)} \big[ Q(s,a) - \alpha \log \pi(a|s) \big],\\
& \text{s.t.} \; \mathbb{E}_{s\sim \mathcal{D}, a\sim \pi(\cdot|s)}
\big[ \log \pi_{\beta}(a|s) \big] \ge \epsilon, \\
& \quad \int_a \pi(a|s)\, da = 1, \ \forall s.
\label{Eq-CCI-Opt-Pro-Appendix}
\end{aligned}
\end{gather}

Solving the above constrained problem Eq.~\eqref{Eq-CCI-Opt-Pro-Appendix} via the Lagrangian, 
\begin{gather} 
\begin{aligned}
\mathcal L(\pi,\lambda,\nu) = \mathbb{E}_{s\sim \mathcal{D}, a\sim \pi(\cdot|s)} \left[ Q(s,a) - \alpha \log \pi(a|s) \right] + \lambda \left[ \mathbb{E}_{s\sim \mathcal{D}, a\sim \pi(\cdot|s)} \left[ \log \pi_{\beta}(a|s) - \epsilon \right] \right] + \mathbb{E}_{s\sim\mathcal D} \left[ \nu(s) \big[ \int_a \pi(a|s)\, da - 1 \big] \right],
\label{Eq-CCI-Lagrangian-Appendix}
\end{aligned}
\end{gather}

To obtain the optimal non-parametric policy, we take the functional derivative of $\mathcal L(\pi,\lambda,\nu)$ w.r.t. $\pi(\cdot|s)$, and setting it to zero yields: 
\begin{gather}
\begin{aligned}
0
& = \frac{\partial \mathcal L(\pi,\lambda,\nu)}{\partial \pi} \\
& = Q(s,a) - \alpha\big(1+\log\pi(a|s)\big) + \lambda\log\pi_\beta(a|s) + \nu(s),
\end{aligned}
\end{gather}
Rearranging gives
\begin{gather}
\log\pi(a|s)
= \frac{1}{\alpha}\Big(Q(s,a)+\lambda\log\pi_\beta(a|s)+\nu(s)-\alpha\Big).
\end{gather}
Exponentiating both sides yields
\begin{gather}
\pi^*_\lambda(a|s)
= \exp\!\left(\frac{\nu(s)-\alpha}{\alpha}\right)\,
\exp\!\left(\frac{Q(s,a)}{\alpha}\right)\,
\pi_\beta(a|s)^{\lambda/\alpha}.
\end{gather}
Since the prefactor $\exp((\nu(s)-\alpha)/\alpha)$ does not depend on $a$, it can be absorbed into the normalizer, giving
\begin{gather}
\pi^*_\lambda(a|s)\ \propto\ \exp\!\left(\frac{Q(s,a)}{\alpha}\right)\,\pi_\beta(a|s)^{\lambda/\alpha}.
\end{gather}
Enforcing $\int \pi^*_\lambda(a|s)\,da = 1$ yields the normalized closed-form solution
\begin{gather} 
\begin{aligned}
\pi^*_{\lambda}(a|s)
= \exp \Big(\frac{Q(s,a)-Z_{\lambda}(s)}{\alpha}\Big)\ \pi_{\beta}(a|s)^{\lambda/\alpha},
\end{aligned}
\label{Eq-CCI-ClosedForm-Z-Appendix}
\end{gather}
where $\lambda\ge 0$ and $\nu(s)$ are dual variables, and $Z_\lambda(s) = \alpha \log \int_a \exp \Big(\frac{Q(s,a)}{\alpha}\Big) \pi_\beta(a|s)^{\lambda/\alpha}\,da$.

Similarly within InAC \cite{InAC}, we also replace $Z_\lambda(s)$ with value function $V(s)$. Therefore, the closed-form solution Eq.~\eqref{Eq-CCI-ClosedForm-Z-Appendix} can be transformed into the following form,
\begin{gather} 
\begin{aligned}
\pi^*_{\lambda}(a|s) \propto \exp \Big(\frac{A(s, a)}{\alpha}\Big)\ \pi_{\beta}(a|s)^{\lambda/\alpha},
\end{aligned}
\label{Eq-CCI-ClosedForm-Appendix}
\end{gather}

Then, we project the non-parametric optimal policy $\pi_{\lambda}^*$ into the parametric space $\pi_{\theta}$ by minimizing the reverse KL divergence under the state distribution in dataset,
\begin{gather} 
\begin{aligned}
\pi_{\theta}(a|s) & = \mathop{\arg\min}_{\pi_{\theta}} \mathbb{E}_{s\sim \mathcal{D}} \big[\mathop{D_{\text{KL}}} \left[\pi^*_{\lambda}(\cdot|s) || \pi_{\theta}(\cdot|s) \right] \big] \\
& = \arg \min_{\theta} \mathbb{E}_{s\sim \mathcal{D}} \Big[ \mathbb{E}_{a \sim \pi^*_{\lambda}(\cdot|s)} \big[\log \pi^*_{\lambda}(a|s) - \log \pi_{\theta}(a|s) \big] \Big] \\
& = \arg \min_{\theta} \mathbb{E}_{s\sim \mathcal{D}} \Big[ \mathbb{E}_{a \sim \pi_{\beta}(\cdot|s)} \Big[ - \exp \Big( \frac{A(s, a)}{\alpha}\Big) \pi_{\beta}(a|s)^{(\lambda - \alpha)/ \alpha} \log \pi_{\theta}(a|s) \Big] \Big] \\
& = \arg \max_{\theta} \mathbb{E}_{s\sim \mathcal{D}} \Big[ \mathbb{E}_{a \sim \pi_{\beta}(\cdot|s)} \Big[ \exp \Big( \frac{A(s, a)}{\alpha} + \frac{\lambda - \alpha}{\alpha} \log \pi_{\beta}(a|s) \Big) \log \pi_{\theta}(a|s) \Big] \Big].
\label{Eq-CCI-Reverse-KL-Appendix}
\end{aligned}
\end{gather}

From Eq.~\eqref{Eq-CCI-Reverse-KL-Appendix}, the policy improvement objective can be written in the dataset form
\begin{gather}
\max_{\theta}\ \mathcal{J}^{\text{CCI}}(\theta;\lambda) = 
\mathbb{E}_{(s,a)\sim\mathcal D}\Big[w_\lambda(s,a)\,\log\pi_\theta(a|s)\Big],
\label{Eq-CCI-weight-unified}
\end{gather}
where $w_\lambda(s,a) = 
\exp \Big( \frac{A(s, a)}{\alpha} + \frac{\lambda - \alpha}{\alpha} \log \pi_{\beta}(a|s) \Big)$. This yields the CCI policy improvement objective Eq.~\eqref{Eq-CCI-Obj}. 

Varying $\lambda$ continuously interpolates the influence of the weight term $w_\lambda(s,a)$, yielding several representative regimes:

\paragraph{\ding{172} Support Constraint $\lambda=0$.}
When $\lambda=0$, the objective becomes
\begin{gather}
\max_{\theta}\ \mathcal{J}^{\text{CCI}}(\theta; 0) = 
\mathbb{E}_{(s,a)\sim\mathcal D}\Big[ \exp \Big( \frac{A(s, a)}{\alpha} - \log \pi_{\beta}(a|s) \Big) \log\pi_\theta(a|s)\Big],
\end{gather}
which matches policy improvement objective used by InAC~\cite{InAC}. 

\paragraph{\ding{173} Support-to-Density Constraint $0 < \lambda < \alpha$.} when $0 < \lambda < \alpha$, the objective becomes
\begin{gather}
\max_{\theta}\ \mathcal{J}^{\text{CCI}}(\theta; \alpha) = 
\mathbb{E}_{(s,a)\sim\mathcal D}\Big[ \exp \Big( \frac{A(s, a)}{\alpha} + \frac{\alpha - \lambda}{\alpha} \log \pi_\beta(a|s) \Big) \log\pi_\theta(a|s)\Big],
\end{gather}
As $\lambda$ increases toward $\alpha$, the update gradually shifts from support constraint toward KL-based density constraint.

\paragraph{\ding{174} KL-based Density Constraint $\lambda = \alpha$.} 
When $\lambda = \alpha$, the objective becomes
\begin{gather}
\max_{\theta}\ \mathcal{J}^{\text{CCI}}(\theta; \alpha) = 
\mathbb{E}_{(s,a)\sim\mathcal D}\Big[ \exp \Big( \frac{A(s, a)}{\alpha} \Big) \log\pi_\theta(a|s)\Big],
\end{gather}
which matches the policy improvement objective in AWAC \cite{AWAC}. 

\paragraph{\ding{175} Density-to-wBC Constraint $\lambda > \alpha$.} When $\lambda > \alpha$, the objective becomes
\begin{gather}
\max_{\theta}\ \mathcal{J}^{\text{CCI}}(\theta;\lambda)
=\mathbb{E}_{(s,a)\sim\mathcal D}\Big[
\exp\!\Big(
\underbrace{\tfrac{A(s,a)}{\alpha}}_{\text{KL-density term}}
+
\underbrace{\tfrac{\lambda-\alpha}{\alpha}\log \pi_{\beta}(a|s)}_{\text{wBC term}}
\Big)\ \log\pi_\theta(a|s)
\Big].
\end{gather}

\paragraph{\ding{176} Practical wBC $\lambda \gg \alpha + A(s,a) (\log \pi_{\beta}(a|s))^{-1}$.} When $\lambda \gg \alpha + A(s,a) (\log \pi_{\beta}(a|s))^{-1}$, then $\frac{\lambda - \alpha}{\alpha} \log \pi_{\beta}(a|s) \gg \frac{A(s, a)}{\alpha}$, the objective becomes 
\begin{gather}
\max_{\theta}\ \mathcal{J}^{\text{CCI}}(\theta; \lambda) = 
\mathbb{E}_{(s,a)\sim\mathcal D}\Big[ \exp \big( \frac{\lambda - \alpha}{\alpha} \log \pi_{\beta}(a|s) \big) \log \pi_\theta(a|s) \Big],
\end{gather}

We note that the practical wBC regime corresponds to choosing
$ \lambda \gg \alpha + A(s,a)\big(\log \pi_{\beta}(a|s)\big)^{-1}.
$ Next, we provide an explicit upper bound on $\big|A(s,a)\,(\log \pi_{\beta}(a|s))^{-1}\big|$ under the maximum-entropy setting.

In practice, we stabilize training by clipping log-probabilities for any stochastic policy: $\log\pi(a|s)\in[\ell_{\min},\ell_{\max}]$. Define
$$C_{\max} \triangleq \max\{|\ell_{\min}|,|\ell_{\max}|\} 
\Rightarrow 
|-\alpha\log\pi(a|s)|\le \alpha C_{\max},$$
In common RL reward setting, the reward function is bounded by $|r(s, a)| \le R_{\max}$. Then the per-step soft reward satisfies 
$ \big|r(s,a)-\alpha\log\pi(a|s)\big| \le R_{\max}+\alpha C_{\max}$. By the discounted sum, we have $|Q^{\pi}(s,a)| \le \frac{R_{\max}+\alpha C_{\max}}{1-\gamma}$, $
|V^{\pi}(s)| \le \frac{R_{\max}+\alpha C_{\max}}{1-\gamma}$. Therefore, 
\begin{gather}
|A^{\pi}(s,a)|
\le |Q^{\pi}(s,a)|+|V^{\pi}(s)|
\le \frac{2\big(R_{\max}+\alpha C_{\max}\big)}{1-\gamma}.
\label{Eq:A_bound_maxent}
\end{gather}
It follows that
\begin{gather}
\left|\frac{A(s,a)}{\log\pi_\beta(a|s)}\right|
\le
\frac{|A(s,a)|}{C^{\beta}_{\min} + \delta}
\le
\frac{2\big(R_{\max}+\alpha C_{\max}\big)}{(1-\gamma)(C^{\beta}_{\min} + \delta)}.
\label{Eq:A_over_logpb_bound_maxent}
\end{gather}

Therefore, a sufficient condition for the practical wBC constraint is to choose $\lambda \gg \alpha + \frac{2\big(R_{\max}+\alpha C_{\max}\big)}{(1-\gamma)(C^{\beta}_{\min} + \delta)}$, where $C^{\beta}_{\min} \triangleq \min |\log\pi_\beta(a|s)|$, and $\delta$ is a small constant introduced to avoid division by zero.

Summarizing the above, we can obtain Tab.~\ref{Tab-CCI-Lambda}.

\emph{In addition, \emph{SAC \cite{SAC} is another extreme initiation of the CCI framework, minimizing the forward KL divergence $\mathop{D_{\text{KL}}} \left[\pi_{\theta}(\cdot|s) || \pi^*_{\lambda}(\cdot|s) \right]$, and setting $\lambda=0$, then the policy improvement objective is }
\begin{gather}
\max_{\theta}\ \mathcal{J}^{\text{SAC}}(\theta; 0) = 
\mathbb{E}_{a \sim \pi_{\theta}(\cdot|s)}\big[ \exp \big( \frac{A(s, a)}{\alpha} - \log \pi_{\theta}(a|s) \big)\big],
\end{gather}}
which matches the policy improvement objective in SAC \cite{SAC}.

\section{Proof for Proposition \ref{Proposition}}
\label{Appendix-Prop-lambda-monotone}

\begin{proposition}
\label{prop:lambda-monotone-appendix}
Fix any state $s$ and temperature $\alpha>0$.
Let $\pi^*_{\lambda}(\cdot|s)$ denote the closed-form optimizer in Eq.~\eqref{Eq-CCI-ClosedForm} for a given $\lambda\ge 0$.
Define
\begin{equation}
g_s(\lambda)
:= \mathbb{E}_{a\sim \pi^*_{\lambda}(\cdot|s)}\!\big[\log \pi_\beta(a|s)\big].
\label{eq:gs-def-appendix}
\end{equation}
Then $g_s(\lambda)$ is monotone non-decreasing in $\lambda$, and
\begin{equation}
\frac{d}{d\lambda} g_s(\lambda)
= \frac{1}{\alpha}\,
\mathrm{Var}_{a\sim \pi^*_{\lambda}(\cdot|s)}\!\big(\log \pi_\beta(a|s)\big)
\;\ge\; 0.
\label{eq:gs-deriv-appendix}
\end{equation}
Consequently, $g(\lambda):=\mathbb{E}_{s\sim \mathcal{D}}[g_s(\lambda)]$ also satisfies $g'(\lambda)\ge 0$.
\end{proposition}

\begin{proof}
Fix $s$ and abbreviate $Q(a):=Q(s,a)$ and $\pi_\beta(a):=\pi_\beta(a|s)$.
Define the unnormalized density and the partition function
\begin{equation}
w_\lambda(a):=\exp\!\left(\frac{Q(a)}{\alpha}\right)\,\pi_\beta(a)^{\lambda/\alpha},
\qquad
\mathcal{Z}(\lambda):=\int w_\lambda(a)\,da,
\label{eq:w-and-Z-appendix}
\end{equation}

so that $\pi^*_\lambda(a|s)=w_\lambda(a)/\mathcal{Z}(\lambda)$.
By definition,
\begin{equation}
g_s(\lambda)
=\int \pi^*_\lambda(a|s)\,\log \pi_\beta(a)\,da.
\label{eq:gs-int-appendix}
\end{equation}
Differentiating Eq.~\eqref{eq:gs-int-appendix} w.r.t.\ $\lambda$ and exchanging derivative and integral,
\begin{equation}
\frac{d}{d\lambda}g_s(\lambda)
=\int \frac{\partial}{\partial\lambda}\pi^*_\lambda(a|s)\,\log\pi_\beta(a)\,da.
\label{eq:gs-deriv-start-appendix}
\end{equation}
Since $\pi^*_\lambda=w_\lambda/\mathcal{Z}(\lambda)$, we have
\begin{align}
\frac{\partial}{\partial\lambda}\pi^*_\lambda(a|s)
&=\frac{w'_\lambda(a)}{\mathcal{Z}(\lambda)}-\frac{w_\lambda(a)\mathcal{Z}'(\lambda)}{\mathcal{Z}(\lambda)^2}
=\pi^*_\lambda(a|s)\Big(\frac{w'_\lambda(a)}{w_\lambda(a)}-\frac{\mathcal{Z}'(\lambda)}{\mathcal{Z}(\lambda)}\Big).
\label{eq:pi-derivative-general-appendix}
\end{align}
Using $w_\lambda(a)=\exp(Q(a)/\alpha)\exp((\lambda/\alpha)\log\pi_\beta(a))$, we obtain
\begin{equation}
\frac{w'_\lambda(a)}{w_\lambda(a)}=\frac{1}{\alpha}\log\pi_\beta(a).
\label{eq:w-ratio-appendix}
\end{equation}
Moreover,
\begin{equation}
\mathcal{Z}'(\lambda)=\int w'_\lambda(a)\,da=\int w_\lambda(a)\cdot \frac{1}{\alpha}\log\pi_\beta(a)\,da,
\qquad\Rightarrow\qquad
\frac{\mathcal{Z}'(\lambda)}{\mathcal{Z}(\lambda)}
=\frac{1}{\alpha}\int \pi^*_\lambda(a|s)\log\pi_\beta(a)\,da
=\frac{1}{\alpha}g_s(\lambda).
\label{eq:Z-ratio-appendix}
\end{equation}
Substituting Eq.~\eqref{eq:w-ratio-appendix}, \eqref{eq:Z-ratio-appendix} into Eq.~\eqref{eq:pi-derivative-general-appendix} yields
\begin{equation}
\frac{\partial}{\partial\lambda}\pi^*_\lambda(a|s)
=\frac{1}{\alpha}\pi^*_\lambda(a|s)\Big(\log\pi_\beta(a)-g_s(\lambda)\Big).
\label{eq:pi-derivative-final-appendix}
\end{equation}
Plugging Eq.~\eqref{eq:pi-derivative-final-appendix} into Eq.~\eqref{eq:gs-deriv-start-appendix} gives
\begin{align*}
\frac{d}{d\lambda}g_s(\lambda)
&=\frac{1}{\alpha}\int \pi^*_\lambda(a|s)\Big(\log\pi_\beta(a)-g_s(\lambda)\Big)\log\pi_\beta(a)\,da\\
&=\frac{1}{\alpha}\Big(\mathbb{E}_{\pi^*_\lambda}\big[(\log\pi_\beta)^2\big]-\mathbb{E}_{\pi^*_\lambda}[\log\pi_\beta]^2\Big)
=\frac{1}{\alpha}\mathrm{Var}_{a\sim \pi^*_\lambda(\cdot|s)}\big(\log\pi_\beta(a)\big)\ \ge\ 0,
\end{align*}
which proves Eq.~\eqref{eq:gs-deriv-appendix}. The statement for $g(\lambda)=\mathbb{E}_{s\sim\mathcal{D}}[g_s(\lambda)]$ follows by linearity of expectation.
\end{proof}

\section{Proof for Performance Lower Bound}
\label{Appendix-Per-Lower-Bound}

\begin{lemma}[Maximum-Entropy Performance Difference Lemma]
\label{lem-soft-pdl-appendix}
Let $\pi$ and $\pi_\beta$ be two stationary policies. Consider the maximum-entropy performance
$
\mathcal{J}(\pi)
=
\mathbb{E}_{\tau \sim \pi}
\left[
\sum_{t=0}^{\infty}\gamma^t
\big(r(s_t,a_t)-\alpha\log\pi(a_t|s_t)\big)
\right].
$
Then the maximum-entropy performance difference satisfies
\begin{gather}
\begin{aligned}
\mathcal{J}(\pi)-\mathcal{J}(\pi_\beta)
=
\frac{1}{1-\gamma}\,
\mathbb{E}_{s\sim d_\pi}
\Big[
\mathbb{E}_{a\sim\pi(\cdot|s)}\big[\tilde A^{\pi_\beta}(s,a)\big]
-\alpha\, D_{\mathrm{KL}}\!\big(\pi(\cdot|s)\,\|\,\pi_\beta(\cdot|s)\big)
\Big],
\end{aligned}
\label{Eq-soft-pdl-main-appendix}
\end{gather}
where $\tilde A^{\pi_{\beta}} (s, a)$ is the advantage function under the shaped reward 
$\tilde r(s, a) = r(s,a)-\alpha\log\pi_\beta(a|s)$.
\end{lemma}

\begin{proof}
Introduce $\pi_\beta$ into the entropy-regularized return:
\begin{align}
\mathcal{J}(\pi)
&=
\mathbb{E}_{\tau\sim\pi}
\Big[
\sum_{t=0}^{\infty} \gamma^t
\big(r(s_t,a_t) - \alpha \log\pi(a_t|s_t)\big)
\Big] \nonumber\\
&=
\mathbb{E}_{\tau\sim\pi}
\Big[
\sum_{t=0}^{\infty} \gamma^t
\big(r(s_t,a_t) - \alpha\log\pi_\beta(a_t|s_t)\big)
\Big]
-\alpha\,\mathbb{E}_{\tau\sim\pi}
\Big[
\sum_{t=0}^{\infty} \gamma^t
\log\frac{\pi(a_t|s_t)}{\pi_\beta(a_t|s_t)}
\Big].
\label{Eq-soft-pdl-step1-appendix}
\end{align}

Define the shaped return
$\widetilde{\mathcal{J}}(\pi):=\mathbb{E}_{\tau\sim\pi}\big[\sum_{t=0}^{\infty} \gamma^t\,\tilde r(s_t,a_t)\big]$
with $\tilde r(s, a) = r(s,a)-\alpha\log\pi_\beta(a|s)$.
Since $\log\frac{\pi_\beta(a|s)}{\pi_\beta(a|s)}=0$, we have $\mathcal{J}(\pi_\beta)=\widetilde{\mathcal{J}}(\pi_\beta)$ and thus
\begin{equation}
\mathcal{J}(\pi)-\mathcal{J}(\pi_\beta)
=
\big(\widetilde{\mathcal{J}}(\pi)-\widetilde{\mathcal{J}}(\pi_\beta)\big)
-\alpha\,\mathbb{E}_{\tau\sim\pi}
\Big[
\sum_{t=0}^{\infty} \gamma^t
\log\frac{\pi(a_t|s_t)}{\pi_\beta(a_t|s_t)}
\Big].
\label{Eq-soft-pdl-step2-appendix}
\end{equation}

For the first term, $\widetilde{\mathcal{J}}$ is a standard discounted return with reward $\tilde r$.
By the performance difference theorem (e.g., Lemma~6.1 in~\cite{PDT}),
\begin{equation}
\widetilde{\mathcal{J}}(\pi)-\widetilde{\mathcal{J}}(\pi_\beta)
=
\frac{1}{1-\gamma}\,
\mathbb{E}_{s\sim d_\pi}\,
\mathbb{E}_{a\sim\pi(\cdot|s)}\big[\tilde A^{\pi_\beta}(s,a)\big].
\label{Eq-soft-pdl-step3-appendix}
\end{equation}

For the second term, using the discounted occupancy identity
$d_\pi(s)=(1-\gamma)\sum_{t=0}^{\infty}\gamma^t\,\Pr(s_t=s\mid\pi)$, we obtain
\begin{align}
\mathbb{E}_{\tau\sim\pi}
\Big[
\sum_{t=0}^{\infty} \gamma^t
\log \frac{\pi(a_t|s_t)}{\pi_\beta(a_t|s_t)}
\Big]
&=
\frac{1}{1-\gamma}\,
\mathbb{E}_{s\sim d_\pi}\,
\mathbb{E}_{a \sim \pi(\cdot|s)}
\Big[
\log\frac{\pi(a|s)}{\pi_\beta(a|s)}
\Big] \nonumber\\
&=
\frac{1}{1-\gamma}\,
\mathbb{E}_{s \sim d_\pi}
\Big[
D_{\mathrm{KL}}\!\big(\pi(\cdot|s)\,\|\, \pi_\beta(\cdot|s)\big)
\Big].
\label{Eq-soft-pdl-step4-appendix}
\end{align}

Substituting Eq.~\eqref{Eq-soft-pdl-step3-appendix} and Eq.~\eqref{Eq-soft-pdl-step4-appendix} into Eq.~\eqref{Eq-soft-pdl-step2-appendix} yields Eq.~\eqref{Eq-soft-pdl-main-appendix}.
\end{proof}

\begin{theorem}[Performance Lower Bound for the Optimal Policy]
\label{thm-optimal-policy-lower-bound-appendix}
Define $\epsilon_\beta(\pi) := \max_{s} \Big|\mathbb{E}_{a\sim\pi(\cdot|s)}\big[\tilde A^{\pi_\beta}(s,a)\big]\Big|$, $\kappa_{\beta}(\pi) := \max_{s} \mathrm{KL}\big(\pi(\cdot|s) \| \pi_\beta(\cdot|s)\big)$ and the average total variation distance under $\nu$ by $\Delta_{\mathrm{TV}} (\pi, \pi_\beta; \nu) := \mathbb{E}_{s\sim \nu} \big[D_{\mathrm{TV}}(\pi(\cdot|s), \pi_\beta(\cdot|s))\big]$. Assume $|\log \pi_\beta(a|s)| \le B$ for all $(s, a)$. For a given $\lambda\ge 0$ and $\alpha>0$, $\pi^*_{\lambda}$ satisfies the following maximum-entropy performance lower bound:
\begin{gather}
\begin{aligned}
\mathcal{J}(\pi^*_\lambda) \ge \mathcal{J}(\pi_\beta) - \frac{2 B |\lambda - \alpha|}{1-\gamma} \Delta_{\mathrm{TV}} (\pi^*_\lambda,\pi_\beta; d_{\pi_\beta}) - \frac{4 \alpha B + 2 \gamma \Big(\epsilon_\beta (\pi^*_\lambda) + \alpha \kappa_{\beta} (\pi^*_\lambda)\Big)} {(1- \gamma)^2} 
\Delta_{\mathrm{TV}} (\pi^*_\lambda,\pi_\beta; d_{\pi_\beta}).
\end{aligned}
\label{Eq-Opt-Policy-LB-appendix}
\end{gather}
\end{theorem}

\begin{proof}
By the maximum-entropy performance difference lemma (Lemma~\ref{lem-soft-pdl-appendix}),
\begin{equation}
\mathcal{J}(\pi)-\mathcal{J}(\pi_\beta)
=
\frac{1}{1-\gamma}\,
\mathbb{E}_{s\sim d_\pi}\Big[
\mathbb{E}_{a\sim\pi(\cdot|s)}\big[\tilde A^{\pi_\beta}(s,a)\big]
-\alpha\,D_{\mathrm{KL}}\big(\pi(\cdot|s)\,\|\,\pi_\beta(\cdot|s)\big)
\Big].
\label{Eq-soft-pdl-used-appendix}
\end{equation}
Let
\[
f_\pi(s):=\mathbb{E}_{a\sim\pi(\cdot|s)}\big[\tilde A^{\pi_\beta}(s,a)\big]
-\alpha\,D_{\mathrm{KL}}\big(\pi(\cdot|s)\,\|\,\pi_\beta(\cdot|s)\big),
\]
so that the RHS of~\eqref{Eq-soft-pdl-used-appendix} equals $\frac{1}{1-\gamma}\mathbb{E}_{d_\pi}[f_\pi(s)]$.
Decomposing around $d_{\pi_\beta}$ yields
\begin{equation}
\mathbb{E}_{d_\pi}[f_\pi] = \mathbb{E}_{d_{\pi_\beta}}[f_\pi]
+\Big(\mathbb{E}_{d_\pi}[f_\pi]-\mathbb{E}_{d_{\pi_\beta}}[f_\pi]\Big).
\label{Eq-occupancy-decomp-me-appendix}
\end{equation}

By definition of $\epsilon_{\beta}(\pi)$ and $\kappa_{\beta}(\pi)$, for all $s$ we have
\[
|f_\pi(s)|
\le
\Big|\mathbb{E}_{a\sim\pi(\cdot|s)}\big[\tilde A^{\pi_\beta}(s,a)\big]\Big|
+\alpha\,D_{\mathrm{KL}}\big(\pi(\cdot|s) \| \pi_\beta(\cdot|s)\big)
\le
\epsilon_\beta(\pi)+\alpha \kappa_\beta(\pi),
\]
hence
\begin{equation}
\Big|\mathbb{E}_{d_\pi}[f_\pi]-\mathbb{E}_{d_{\pi_\beta}}[f_\pi]\Big|
\le \|d_\pi-d_{\pi_\beta}\|_1\Big(\epsilon_\beta(\pi)+\alpha \kappa_\beta(\pi)\Big).
\label{Eq-deviation-bound-me-appendix}
\end{equation}

Moreover, the divergence between discounted state visitation distributions (Lemma~3 in~\cite{CPO}) gives
\begin{equation}
\|d_\pi-d_{\pi_\beta}\|_1
\le \frac{2\gamma}{1-\gamma}\,
\Delta_{\mathrm{TV}}(\pi,\pi_\beta; d_{\pi_\beta}).
\label{Eq-occupancy-tv-bound-me-appendix}
\end{equation}
Combining~\eqref{Eq-deviation-bound-me-appendix} and~\eqref{Eq-occupancy-tv-bound-me-appendix} yields
\begin{equation}
\mathbb{E}_{d_\pi}[f_\pi]
\ge
\mathbb{E}_{d_{\pi_\beta}}[f_\pi]
-\frac{2\gamma}{1-\gamma}\Big(\epsilon_\beta(\pi)+\alpha \kappa_\beta(\pi)\Big)
\Delta_{\mathrm{TV}}(\pi,\pi_\beta; d_{\pi_\beta}).
\label{Eq-tv-pessimism-me-appendix}
\end{equation}
Substituting~\eqref{Eq-tv-pessimism-me-appendix} into~\eqref{Eq-soft-pdl-used-appendix} and setting $\pi=\pi^*_\lambda$ yields
\begin{equation}
\mathcal{J}(\pi^*_\lambda)-\mathcal{J}(\pi_\beta)
\ge
\frac{1}{1-\gamma} \mathbb{E}_{s\sim d_{\pi_\beta}} f_{\pi^*_\lambda}(s)
-\frac{2\gamma}{(1 - \gamma)^2} \Big( \epsilon_\beta(\pi^*_\lambda) + \alpha \kappa_\beta (\pi^*_\lambda)\Big)
\Delta_{\mathrm{TV}}(\pi^*_\lambda,\pi_\beta; d_{\pi_\beta}).
\label{Eq-step1-bound-me-appendix}
\end{equation}

For the first term,
\[
\mathbb{E}_{s\sim d_{\pi_\beta}} f_{\pi^*_\lambda}(s)
=
\mathbb{E}_{s\sim d_{\pi_\beta}} \Big[
\mathbb{E}_{a \sim \pi^*_\lambda(\cdot|s)} \tilde A^{\pi_\beta}(s,a)
-\alpha\,D_{\mathrm{KL}} \big( \pi^*_\lambda(\cdot|s) \| \pi_\beta(\cdot|s)\big)
\Big].
\]

Next, define the auxiliary objective
\[
F(\pi):=\mathbb{E}_{s\sim d_{\pi_\beta}}
\Big[
\mathbb{E}_{a\sim\pi(\cdot|s)} A^{\pi_\beta}(s,a)
-\alpha\,D_{\mathrm{KL}}\big(\pi(\cdot|s) \| \pi_\beta(\cdot|s)\big)
+ (\lambda-\alpha)\,\mathbb{E}_{a\sim\pi(\cdot|s)} \log\pi_\beta(a|s)
\Big],
\]
where $A^{\pi_\beta}(s,a)$ denotes the advantage under the original reward $r(s,a)$.
Since $\pi^*_\lambda$ maximizes $F$ in the CCI optimization problem, we have $F(\pi^*_\lambda)\ge F(\pi_\beta)$.
For $\pi_\beta$, note that
$\mathbb{E}_{a\sim\pi_\beta(\cdot|s)}[A^{\pi_\beta}(s,a)]=0$ and
$D_{\mathrm{KL}}(\pi_\beta(\cdot|s)\|\pi_\beta(\cdot|s))=0$, hence
\[
F(\pi_\beta)=(\lambda-\alpha)\,g_{d_{\pi_\beta}}(\pi_\beta),
\qquad
g_{d_{\pi_\beta}}(\pi):=\mathbb{E}_{s\sim d_{\pi_\beta},\,a\sim\pi(\cdot|s)}[\log\pi_\beta(a|s)].
\]
Rearranging gives
\begin{equation}
\mathbb{E}_{s\sim d_{\pi_\beta}}
\Big[
\mathbb{E}_{a\sim\pi^*_\lambda(\cdot|s)} A^{\pi_\beta}(s,a)
-\alpha\,D_{\mathrm{KL}}\big(\pi^*_\lambda(\cdot|s)\,\|\,\pi_\beta(\cdot|s)\big)
\Big]
\ge
-(\lambda-\alpha)\big(g_{d_{\pi_\beta}}(\pi^*_\lambda)-g_{d_{\pi_\beta}}(\pi_\beta)\big).
\label{Eq-surrogate-lb-step-cci-appendix}
\end{equation}

We now relate $\tilde A^{\pi_\beta}$ to $A^{\pi_\beta}$.
Let $c(s,a):=-\alpha\log\pi_\beta(a|s)$ and $\tilde r(s,a)=r(s,a)+c(s,a)$.
By linearity of value functions in reward,
\[
\tilde A^{\pi_\beta}(s,a)=A^{\pi_\beta}(s,a)+A^{\pi_\beta}_{c}(s,a),
\]
where $A^{\pi_\beta}_{c}$ is the advantage under the reward $c$ alone.
Since $|\log\pi_\beta(a|s)|\le B$, we have $|c(s,a)|\le \alpha B$, which implies
$\|A^{\pi_\beta}_{c}\|_\infty \le \frac{2\alpha B}{1-\gamma}$.
Moreover, $\mathbb{E}_{a\sim\pi_\beta(\cdot|s)}[A^{\pi_\beta}_{c}(s,a)]=0$ for all $s$, hence for any $\pi$ and $s$,
\begin{align}
\Big|\mathbb{E}_{a\sim\pi(\cdot|s)}[A^{\pi_\beta}_{c}(s,a)]\Big|
&=
\Big|\int \big(\pi(a|s)-\pi_\beta(a|s)\big)\,A^{\pi_\beta}_{c}(s,a)\,da\Big| \nonumber\\
&\le 2\|A^{\pi_\beta}_{c}\|_\infty\,D_{\mathrm{TV}}\big(\pi(\cdot|s),\pi_\beta(\cdot|s)\big) \nonumber\\
&\le \frac{4\alpha B}{1-\gamma}\,D_{\mathrm{TV}}\big(\pi(\cdot|s),\pi_\beta(\cdot|s)\big).
\label{Eq-Ac-tv-bound-state-appendix}
\end{align}
Therefore, for $\pi=\pi^*_\lambda$,
\begin{align}
\mathbb{E}_{s\sim d_{\pi_\beta}}
\mathbb{E}_{a\sim\pi^*_\lambda(\cdot|s)}\tilde A^{\pi_\beta}(s,a)
&\ge
\mathbb{E}_{s\sim d_{\pi_\beta}}
\mathbb{E}_{a\sim\pi^*_\lambda(\cdot|s)} A^{\pi_\beta}(s,a)
-\frac{4\alpha B}{1-\gamma}\,
\Delta_{\mathrm{TV}}(\pi^*_\lambda,\pi_\beta; d_{\pi_\beta}).
\label{Eq-tildeA-lower-by-A-appendix}
\end{align}

Combining~\eqref{Eq-surrogate-lb-step-cci-appendix} and~\eqref{Eq-tildeA-lower-by-A-appendix} yields
\begin{align}
\mathbb{E}_{s\sim d_{\pi_\beta}} f_{\pi^*_\lambda}(s)
&\ge
-(\lambda-\alpha)\big(g_{d_{\pi_\beta}}(\pi^*_\lambda)-g_{d_{\pi_\beta}}(\pi_\beta)\big)
-\frac{4\alpha B}{1-\gamma}\,
\Delta_{\mathrm{TV}}(\pi^*_\lambda,\pi_\beta; d_{\pi_\beta}).
\label{Eq-first-term-lb-intermediate-appendix}
\end{align}

Finally, for each state $s$, define
\[
\Delta g_s
:=\mathbb{E}_{a\sim\pi^*_\lambda(\cdot|s)}[\log \pi_\beta(a|s)]
-\mathbb{E}_{a\sim\pi_\beta(\cdot|s)}[\log \pi_\beta(a|s)].
\]
Then
\begin{equation}
|\Delta g_s|
\le
\int \big|\pi^*_\lambda(a|s)-\pi_\beta(a|s)\big|\cdot |\log \pi_\beta(a|s)|\,da
\le
2B\,D_{\mathrm{TV}}\big(\pi^*_\lambda(\cdot|s),\pi_\beta(\cdot|s)\big),
\label{Eq-deltag-state-tv-appendix}
\end{equation}
and averaging over $s\sim d_{\pi_\beta}$ gives
\begin{equation}
\big|g_{d_{\pi_\beta}}(\pi^*_\lambda)-g_{d_{\pi_\beta}}(\pi_\beta)\big|
\le
2B\,\Delta_{\mathrm{TV}}(\pi^*_\lambda,\pi_\beta; d_{\pi_\beta}).
\label{Eq-deltag-tv-avg-appendix}
\end{equation}
Substituting~\eqref{Eq-deltag-tv-avg-appendix} into~\eqref{Eq-first-term-lb-intermediate-appendix} yields
\begin{equation}
\mathbb{E}_{s\sim d_{\pi_\beta}} f_{\pi^*_\lambda}(s)
\ge
-2B|\lambda-\alpha|\,\Delta_{\mathrm{TV}}(\pi^*_\lambda,\pi_\beta; d_{\pi_\beta})
-\frac{4\alpha B}{1-\gamma}\,
\Delta_{\mathrm{TV}}(\pi^*_\lambda,\pi_\beta; d_{\pi_\beta}).
\label{Eq-first-term-lb-final-appendix}
\end{equation}

Plugging~\eqref{Eq-first-term-lb-final-appendix} into~\eqref{Eq-step1-bound-me-appendix} completes the proof and yields Eq.~\eqref{Eq-Opt-Policy-LB-appendix}.
\end{proof}

\begin{theorem}[Performance Lower Bound for a Parametric Policy]
\label{thm-para-policy-lower-bound-appendix}
Define the suboptimality gap of a parametric policy $\pi_\theta$ by
\begin{equation}
\delta^{\mathrm{sub}}_\theta(\lambda)
:=
F_\lambda(\pi^*_\lambda)-F_\lambda(\pi_\theta)
\;\ge\; 0,
\label{Eq-subopt-gap-def-appendix}
\end{equation}
where
$F_\lambda(\pi)
=
\mathbb{E}_{s\sim d_{\pi_\beta}}
\Big[
\mathbb{E}_{a\sim\pi(\cdot|s)} \tilde A^{\pi_\beta}(s,a)
-\alpha\,D_{\mathrm{KL}}\!\big( \pi(\cdot|s) \,\|\, \pi_\beta(\cdot|s)\big)
+ (\lambda - \alpha)\,\mathbb{E}_{a\sim\pi(\cdot|s)} \log \pi_\beta(a|s)
\Big].$

Then $\pi_\theta$ satisfies
\begin{gather}
\begin{aligned}
\mathcal{J}(\pi_\theta) \ge\;
\mathcal{J}(\pi_\beta)
-\frac{2B|\lambda - \alpha|}{1-\gamma}\,
\Delta_{\mathrm{TV}}\!\big(\pi_\theta,\pi_\beta; d_{\pi_\beta}\big)
-\frac{\delta^{\mathrm{sub}}_\theta(\lambda)}{1-\gamma} - \frac{4 \alpha B + 2 \gamma \Big(\epsilon_\beta (\pi_\theta) + \alpha \kappa_{\beta} (\pi_\theta)\Big)} {(1- \gamma)^2}\,
\Delta_{\mathrm{TV}} \big(\pi_\theta,\pi_\beta; d_{\pi_\beta}\big).
\end{aligned}
\label{Eq-Para-Policy-LB-appendix}
\end{gather}
\end{theorem}

\begin{proof}
The derivation follows the same decomposition as in the proof of Theorem~\ref{thm-optimal-policy-lower-bound-appendix}.
Applying Eq.~\eqref{Eq-soft-pdl-used-appendix} with $\pi=\pi_\theta$ and using the same occupancy-shift bound yields
\begin{equation}
\mathcal{J}(\pi_\theta)-\mathcal{J}(\pi_\beta)
\ge
\frac{1}{1-\gamma}\,
\mathbb{E}_{s\sim d_{\pi_\beta}} f_{\pi_\theta}(s)
-\frac{2\gamma}{(1 - \gamma)^2} \Big( \epsilon_\beta(\pi_\theta) + \alpha \kappa_\beta (\pi_\theta)\Big)
\Delta_{\mathrm{TV}}(\pi_\theta,\pi_\beta; d_{\pi_\beta}),
\label{Eq-step1-bound-param-appendix}
\end{equation}
where
\[
f_\pi(s)
=
\mathbb{E}_{a\sim\pi(\cdot|s)}\big[\tilde A^{\pi_\beta}(s,a)\big]
-\alpha\,D_{\mathrm{KL}}\!\big(\pi(\cdot|s)\,\|\,\pi_\beta(\cdot|s)\big).
\]

We first lower bound the shaped-advantage term by the original advantage term, as in Eq.~\eqref{Eq-tildeA-lower-by-A-appendix}:
\begin{equation}
\mathbb{E}_{s\sim d_{\pi_\beta},\,a\sim\pi(\cdot|s)}[\tilde A^{\pi_\beta}(s,a)]
\ge
\mathbb{E}_{s\sim d_{\pi_\beta},\,a\sim\pi(\cdot|s)}[A^{\pi_\beta}(s,a)]
-\frac{4\alpha B}{1-\gamma}\,\Delta_{\mathrm{TV}}(\pi,\pi_\beta; d_{\pi_\beta}).
\label{Eq-tildeA-lb-by-A-param-appendix}
\end{equation}

Next, we convert the suboptimality gap in Eq.~\eqref{Eq-subopt-gap-def-appendix} into a baseline-anchored inequality.
Since $\pi^*_\lambda$ maximizes $F_\lambda$, we have $F_\lambda(\pi^*_\lambda)\ge F_\lambda(\pi_\beta)$.
Combining this with Eq.~\eqref{Eq-subopt-gap-def-appendix} gives
\begin{equation}
F_\lambda(\pi_\theta)
=
F_\lambda(\pi^*_\lambda)-\delta^{\mathrm{sub}}_\theta(\lambda)
\ge
F_\lambda(\pi_\beta)-\delta^{\mathrm{sub}}_\theta(\lambda).
\label{Eq-baseline-anchor-from-subopt-appendix}
\end{equation}
For $\pi_\beta$, $\mathbb{E}_{a\sim\pi_\beta}[A^{\pi_\beta}(s,a)]=0$ and
$D_{\mathrm{KL}}(\pi_\beta\|\pi_\beta)=0$, hence
$F_\lambda(\pi_\beta)=(\lambda-\alpha)\,g_{d_{\pi_\beta}}(\pi_\beta)$ where
\[
g_{d_{\pi_\beta}}(\pi):=\mathbb{E}_{s\sim d_{\pi_\beta},\,a\sim\pi(\cdot|s)}[\log\pi_\beta(a|s)].
\]
Rearranging Eq.~\eqref{Eq-baseline-anchor-from-subopt-appendix} yields
\begin{equation}
\mathbb{E}_{s\sim d_{\pi_\beta},\,a\sim\pi_\theta(\cdot|s)}\!\big[A^{\pi_\beta}(s,a)\big]
-\alpha\,\mathbb{E}_{s\sim d_{\pi_\beta}}\!\Big[D_{\mathrm{KL}}\!\big(\pi_\theta(\cdot|s)\,\|\,\pi_\beta(\cdot|s)\big)\Big]
\ge
-(\lambda-\alpha)\Big(g_{d_{\pi_\beta}}(\pi_\theta)-g_{d_{\pi_\beta}}(\pi_\beta)\Big)
-\delta^{\mathrm{sub}}_\theta(\lambda).
\label{Eq-surrogate-lb-param-appendix}
\end{equation}

Finally, we bound the density-shaping difference by total variation.
For each state $s$, define
\[
\Delta g_s
:=
\mathbb{E}_{a\sim\pi_\theta(\cdot|s)}[\log \pi_\beta(a|s)]
-\mathbb{E}_{a\sim\pi_\beta(\cdot|s)}[\log \pi_\beta(a|s)].
\]
Then
\begin{equation}
|\Delta g_s|
\le 2B\,D_{\mathrm{TV}}\big(\pi_\theta(\cdot|s), \pi_\beta(\cdot|s)\big),
\label{Eq-deltag-state-tv-param-appendix}
\end{equation}
and averaging over $s\sim d_{\pi_\beta}$ yields
\begin{equation}
\big|g_{d_{\pi_\beta}}(\pi_\theta)-g_{d_{\pi_\beta}}(\pi_\beta)\big|
\le 2B\,\Delta_{\mathrm{TV}}(\pi_\theta,\pi_\beta; d_{\pi_\beta}).
\label{Eq-deltag-tv-avg-param-appendix}
\end{equation}

Combining Eq.~\eqref{Eq-tildeA-lb-by-A-param-appendix}, Eq.~\eqref{Eq-surrogate-lb-param-appendix}, and Eq.~\eqref{Eq-deltag-tv-avg-param-appendix} yields
\begin{align}
\mathbb{E}_{s\sim d_{\pi_\beta}} f_{\pi_\theta}(s)
&=
\mathbb{E}_{s\sim d_{\pi_\beta}}
\Big[
\mathbb{E}_{a\sim\pi_\theta(\cdot|s)}\tilde A^{\pi_\beta}(s,a)
-\alpha\,D_{\mathrm{KL}}\!\big(\pi_\theta(\cdot|s)\,\|\,\pi_\beta(\cdot|s)\big)
\Big] \nonumber\\
&\ge
-2B|\lambda-\alpha|\,\Delta_{\mathrm{TV}}(\pi_\theta,\pi_\beta; d_{\pi_\beta})
-\frac{4\alpha B}{1-\gamma}\,\Delta_{\mathrm{TV}}(\pi_\theta,\pi_\beta; d_{\pi_\beta})
-\delta^{\mathrm{sub}}_\theta(\lambda).
\label{Eq-first-term-lb-param-appendix}
\end{align}

Plugging Eq.~\eqref{Eq-first-term-lb-param-appendix} into Eq.~\eqref{Eq-step1-bound-param-appendix} completes the proof and yields Eq.~\eqref{Eq-Para-Policy-LB-appendix}.
\end{proof}

\section{Experimental Details}
\label{Appendix-Experiment}

\subsection{Implementation Details for D4RL Experiments}
The vast majority of hyperparameters used in our D4RL experiments are summarized in Tab.~\ref{Tab-D4RL-Hyperparameters}.

\begin{table}[!t]
\caption{D4RL hyperparameters, where $s_{\mathrm{dim}}$ and $a_{\mathrm{dim}}$ are the state and action dimensions, respectively.}
\label{Tab-D4RL-Hyperparameters}
\centering
\small
\setlength{\tabcolsep}{6pt}
\renewcommand{\arraystretch}{1.15}
\begin{tabular}{l|l|l}
\toprule
Hyperparameter & Value & Task \\
\midrule
Actor architecture              & $s_{\mathrm{dim}}$ -- 256 -- 256 -- $2a_{\mathrm{dim}}$ & All \\
$Q$-network                     & $(s_{\mathrm{dim}}{+}a_{\mathrm{dim}})$ -- 256 -- 256 -- 1 & All \\
Value network                   & $s_{\mathrm{dim}}$ -- 256 -- 256 -- 1 & All \\
Behavior network                & $s_{\mathrm{dim}}$ -- 512 -- 512 -- $2a_{\mathrm{dim}}$ & All \\
Activation                      & ReLU & All \\
Learning rate (actor \& critic) & $5e-4 \rightarrow 1e-4$ (CosineAnnealingLR) & All \\
Learning rate (behavior policy) & $1e-4$ & All \\
Learning rate ($\lambda$)       & $1e-5$ & All \\
Target $Q$ soft update rate $\tau$ & $5e-3$ & All \\
\multirow{2}{*}{Evaluation epochs} & 10 & Gym-MuJoCo \& Kitchen \\
                                 & 50 & AntMaze \\
Evaluation frequency             & $5e3$ & All \\
Batch size                       & 256 & All \\
Discount factor $\gamma$         & 0.99 & All \\
Temperature $\alpha$             & 0.1 & All \\
Initial $\lambda$                & 0.1 ($\lambda=\alpha$) & All \\
\multirow{4}{*}{$\epsilon$ in Eq.~\eqref{Eq-CCI-Opt-Pro}} 
& $-5.0$ & Hopper \\
& $-0.2$ & HalfCheetah \& AntMaze \\
& $-0.7$ & Walker2d \\
& $-1.0$ & Kitchen \\
\bottomrule
\end{tabular}
\end{table}

The dataset abbreviations used in Tab.~\ref{Tab-D4RL} and their corresponding full names are listed in Tab.~\ref{Tab-D4RL-Abbreviation}.

\begin{table}[!t]
\caption{Dataset abbreviations for D4RL.}
\label{Tab-D4RL-Abbreviation}
\centering
\small
\setlength{\tabcolsep}{6pt}
\renewcommand{\arraystretch}{1.15}
\begin{tabular}{l|l}
\toprule
Dataset abbreviation & Dataset full name \\
\midrule
halfcheetah-med      & halfcheetah-medium-v2 \\
hopper-med           & hopper-medium-v2 \\
walker2d-med         & walker2d-medium-v2 \\
halfcheetah-med-rep  & halfcheetah-medium-replay-v2 \\
hopper-med-rep       & hopper-medium-replay-v2 \\
walker2d-med-rep     & walker2d-medium-replay-v2 \\
halfcheetah-med-exp  & halfcheetah-medium-expert-v2 \\
hopper-med-exp       & hopper-medium-expert-v2 \\
walker2d-med-exp     & walker2d-medium-expert-v2 \\
\midrule
antmaze-u            & antmaze-umaze-v2 \\
antmaze-u-d          & antmaze-umaze-diverse-v2 \\
antmaze-m-p          & antmaze-medium-play-v2 \\
antmaze-m-d          & antmaze-medium-diverse-v2 \\
antmaze-l-p          & antmaze-large-play-v2 \\
antmaze-l-d          & antmaze-large-diverse-v2 \\
\midrule
kitchen-m            & kitchen-mixed-v0 \\
kitchen-p            & kitchen-partial-v0 \\
\bottomrule
\end{tabular}
\end{table}

Training curves for all D4RL tasks are provided in Fig.~\ref{Fig-D4RL-Curves-Appendix}.

\subsection{Implementation Details for NeoRL2 Experiments}
The vast majority of hyperparameters used in our NeoRL2 experiments are summarized in Tab.~\ref{Tab-NeoRL2-Hyperparameters}.

\begin{table}[!t]
\caption{NeoRL2 hyperparameters. Architectures are denoted by MLP layer sizes, where $s_{\mathrm{dim}}$ and $a_{\mathrm{dim}}$ are the state and action dimensions, respectively.}
\label{Tab-NeoRL2-Hyperparameters}
\centering
\small
\setlength{\tabcolsep}{6pt}
\renewcommand{\arraystretch}{1.15}
\begin{tabular}{l|l|l}
\toprule
Hyperparameter & Value & Task \\
\midrule
Actor architecture              & $s_{\mathrm{dim}}$ -- 256 -- 256 -- $2a_{\mathrm{dim}}$ & All \\
$Q$-network                     & $(s_{\mathrm{dim}}{+}a_{\mathrm{dim}})$ -- 256 -- 256 -- 1 & All \\
Value network                   & $s_{\mathrm{dim}}$ -- 256 -- 256 -- 1 & All \\
Behavior network                & $s_{\mathrm{dim}}$ -- 512 -- 512 -- $2a_{\mathrm{dim}}$ & All \\
Activation                      & ReLU & All \\
Learning rate (actor \& critic) & $5e-4 \rightarrow 1e-4$ (CosineAnnealingLR) & All \\
Learning rate (behavior policy) & $1e-4$ & All \\
Learning rate ($\lambda$)       & $1e-5$ & All \\
Target $Q$ soft update rate $\tau$ & $5e-3$ & All \\
Evaluation epochs               & 10 & All \\
Evaluation frequency            & $5e3$ & All \\
Batch size                      & 256 & All \\
Discount factor $\gamma$        & 0.99 & All \\
Temperature $\alpha$            & 0.5 & All \\
Initial $\lambda$               & 0.5 ($\lambda=\alpha$) & All \\
\multirow{4}{*}{$\epsilon$ in Eq.~\eqref{Eq-CCI-Opt-Pro}} 
& $-5.0$  & Simglucose \& Pipeline \\
& $-10.0$ & SafetyHalfCheetah \& RocketRecovery \& Fusion \\
& $-2.0$  & RandomFrictionHopper \\
& $-0.5$  & DMSD \\
\bottomrule
\end{tabular}
\end{table}

Training curves for all NeoRL2 tasks are provided in Fig.~\ref{Fig-NeoRL2-Curves-Appendix}.

\subsection{Additional Details for Special-Case Comparisons and $\lambda$ Dynamics}
In Section~\ref{Subsec-Threecase}, we keep all hyperparameters consistent with Tab.~\ref{Tab-D4RL-Hyperparameters} and Tab.~\ref{Tab-NeoRL2-Hyperparameters}. We additionally specify the key settings in the main text:
\emph{``For Gym-MuJoCo tasks, we set $\alpha=0.1$. For NeoRL2 tasks, we set $\alpha=0.5$. We instantiate three special cases by choosing: (i) practical wBC: $\lambda=100$ to approximate the wBC regime, (ii) AWAC: $\lambda=\alpha$, and (iii) InAC: $\lambda=0$.''}

\subsection{Additional Details for Behavior Density Estimator Ablation}
\label{Appendix-Behavior-Model}
In the experiments of Subsection~\ref{Subsec-Behavior}, we run inference with the trained Gaussian and CVAE behavior models on the full datasets of \texttt{walker2d-medium-v2} and \texttt{walker2d-medium-replay-v2}, producing the distribution curves shown in Fig.~\ref{Fig-Gaussian_CVAE_log}. In addition to $\xi=0.3$, we also evaluate $\xi \in \{0.1, 0.5, 0.8\}$, with the corresponding curves reported in Fig.~\ref{Fig-Gaussian_CVAE_log-xi-sweep}. The conclusions remain unchanged: compared with the Gaussian estimator, the CVAE likelihood on perturbed OOD actions is not consistently suppressed.

\begin{figure}[!t]
\centering
\begin{minipage}{0.49\linewidth}
  \centering
  \includegraphics[width=\linewidth]{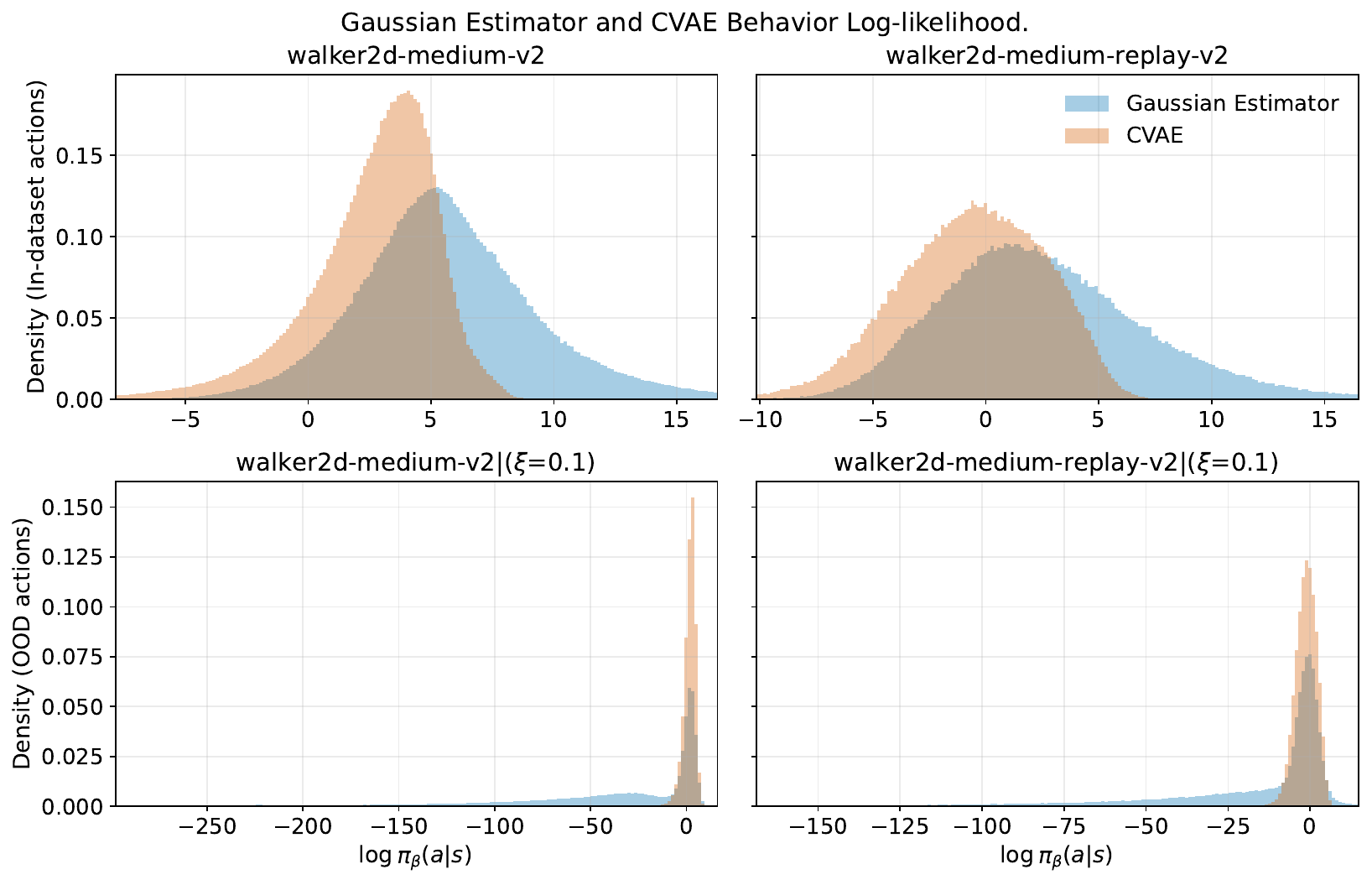}
\end{minipage}\hfill
\begin{minipage}{0.49\linewidth}
  \centering
  \includegraphics[width=\linewidth]{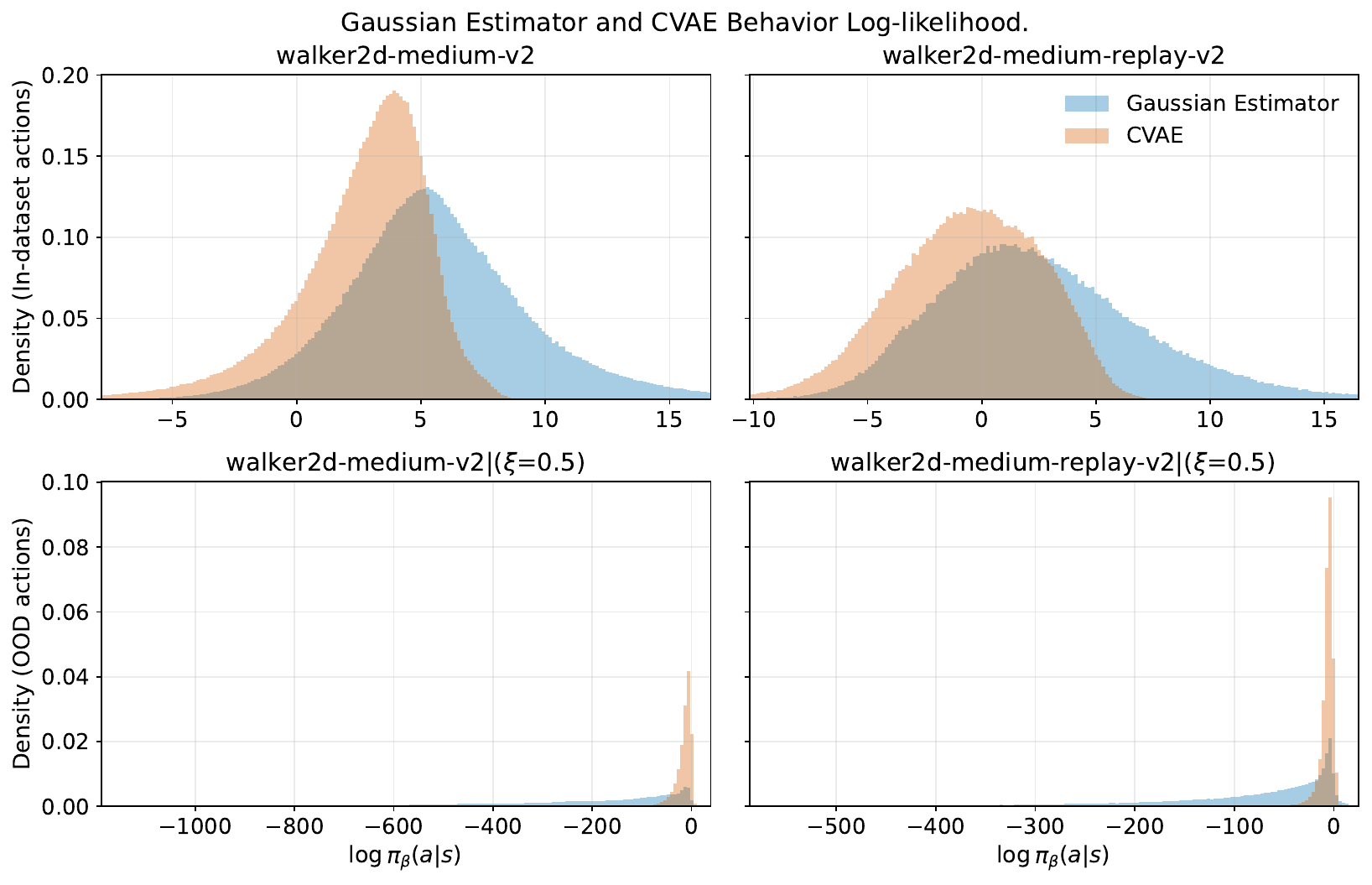}
\end{minipage}

\vspace{2mm}

\begin{minipage}{0.49\linewidth}
  \centering
  \includegraphics[width=\linewidth]{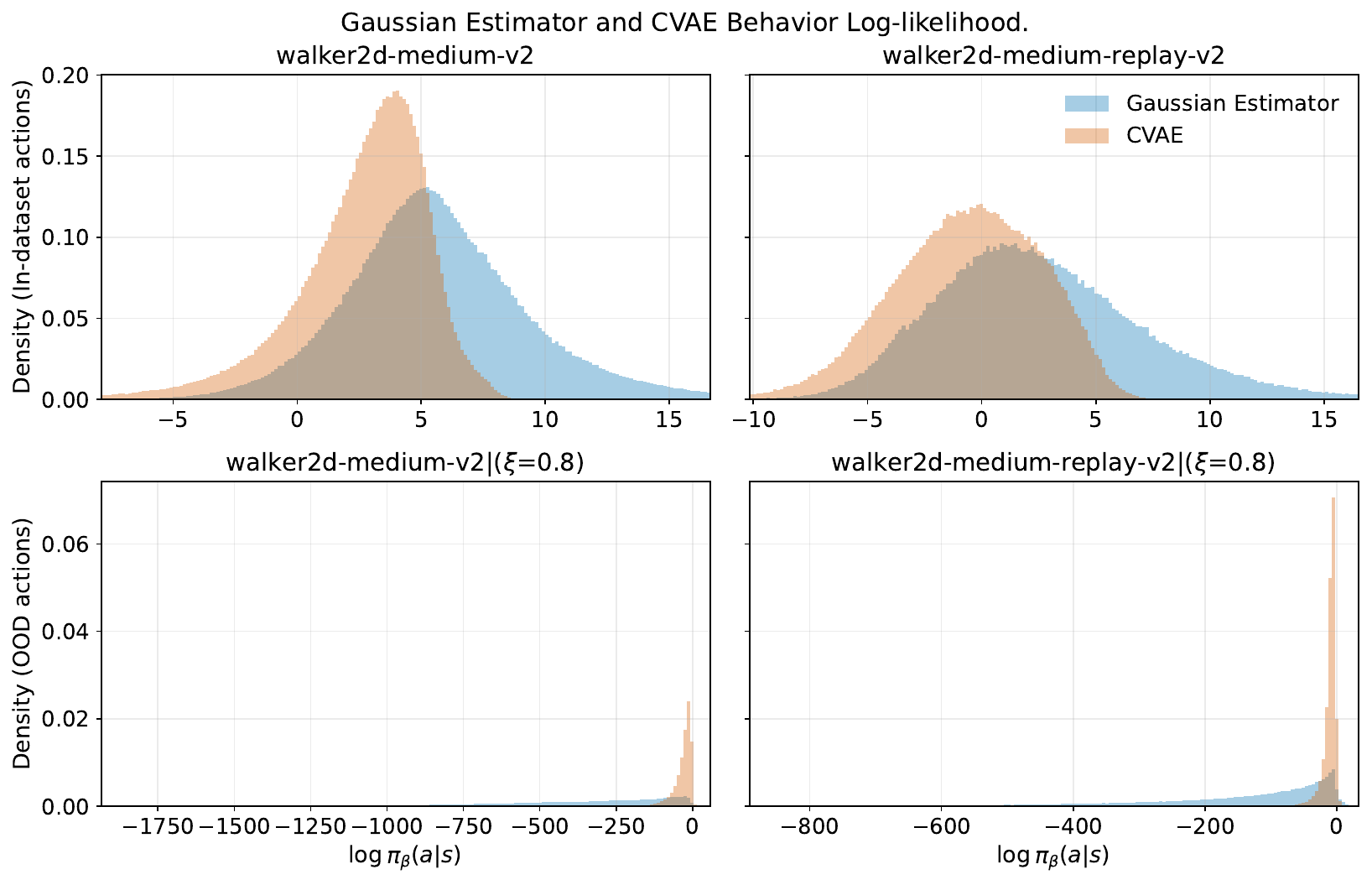}
\end{minipage}

\caption{Gaussian and CVAE behavior models under $\xi \in \{0.1, 0.5, 0.8\}$.}
\label{Fig-Gaussian_CVAE_log-xi-sweep}
\end{figure}

\begin{figure}[!htbp] 
\centerline{\includegraphics[width=1.0\textwidth]{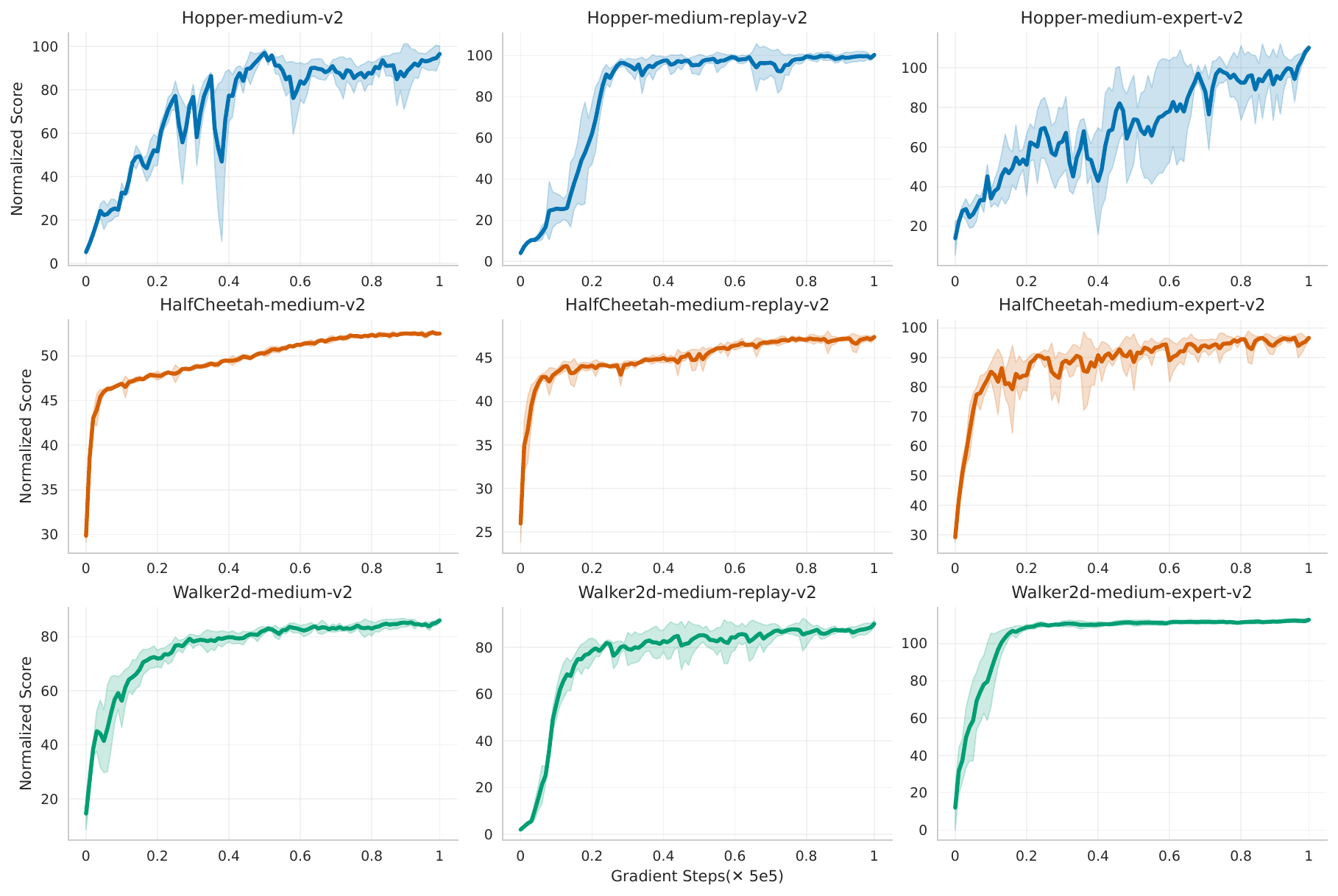}}
\centerline{\includegraphics[width=1.0\textwidth]{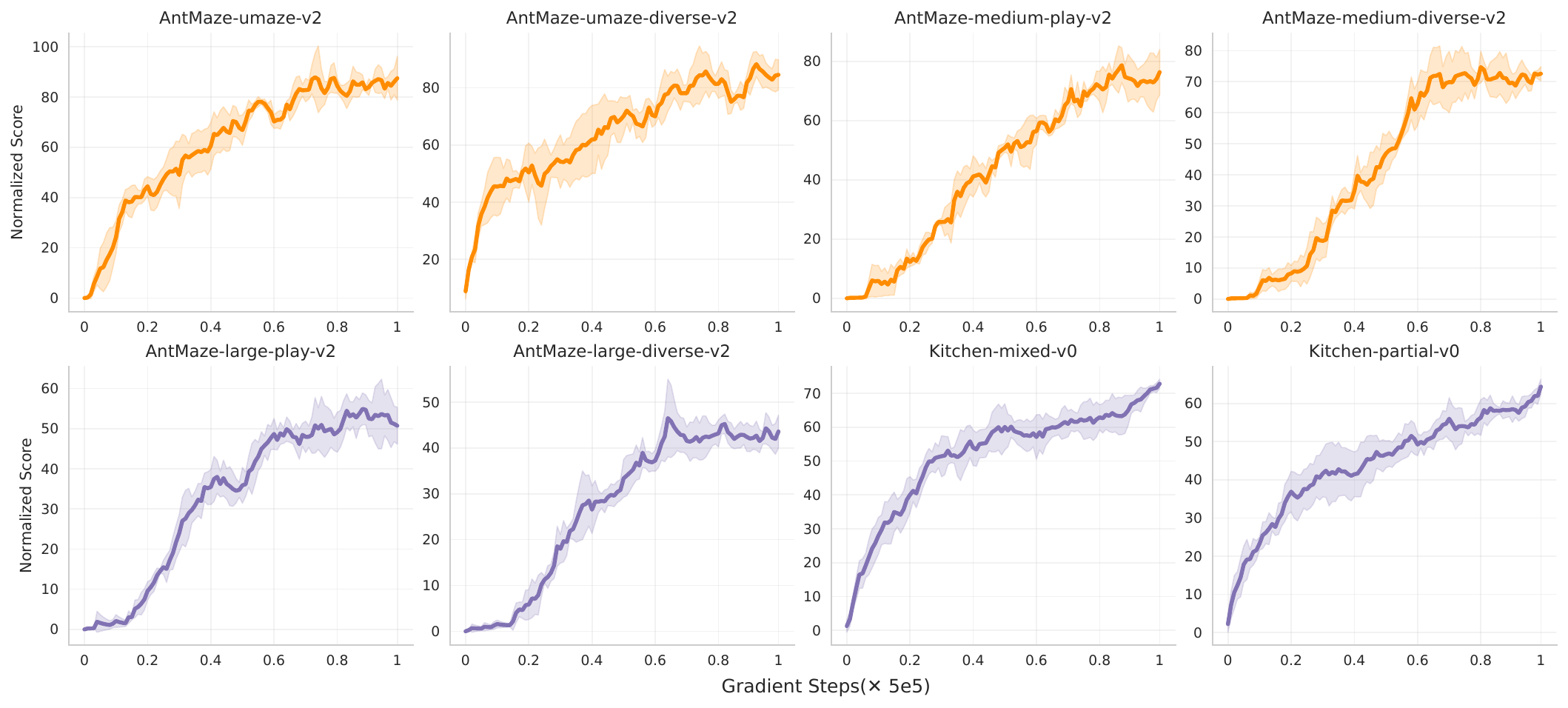}}
\caption{D4RL Learning Curves.}
\label{Fig-D4RL-Curves-Appendix}
\end{figure}

\begin{figure}[!htbp] 
\centerline{\includegraphics[width=1.0\textwidth]{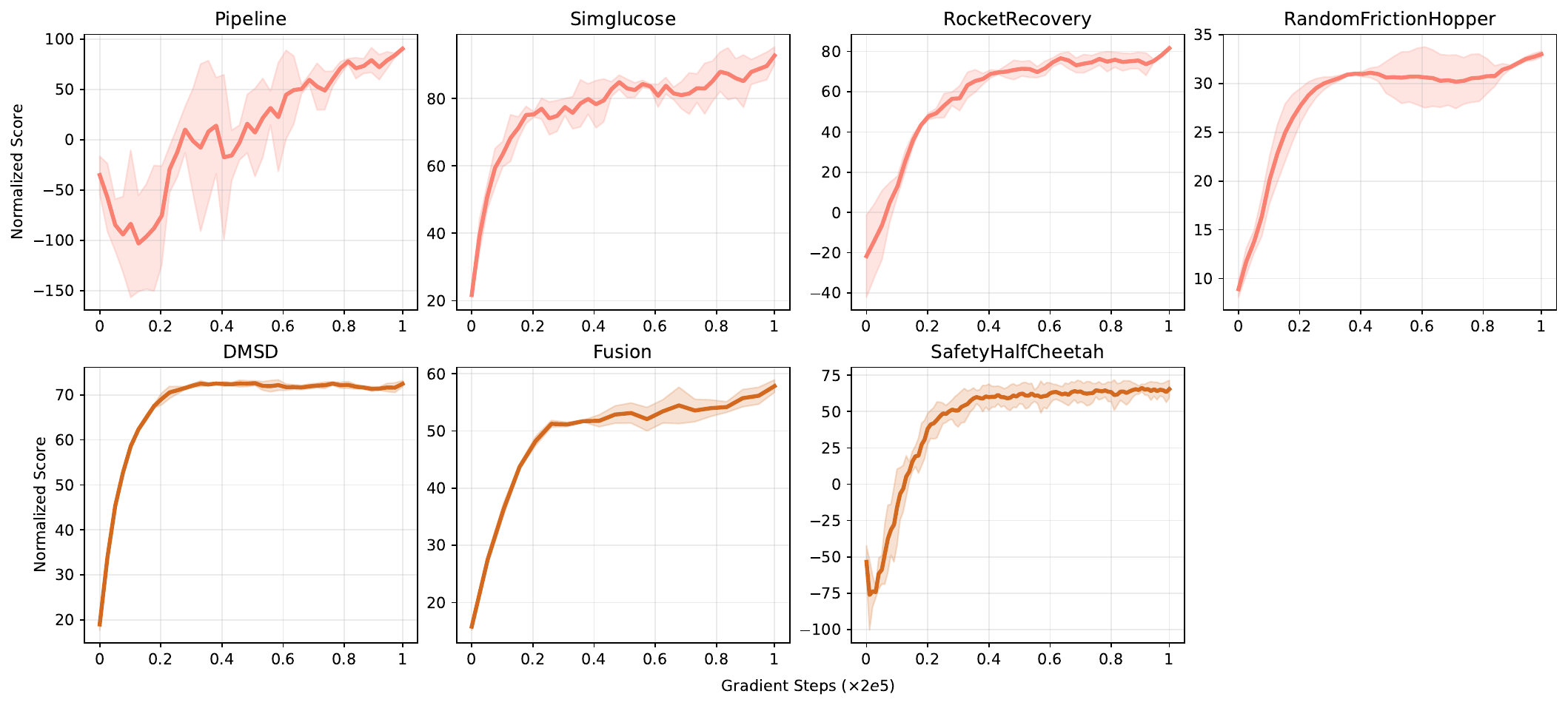}}
\caption{NeoRL2 Learning Curves.}
\label{Fig-NeoRL2-Curves-Appendix}
\end{figure}

\end{document}